%% file: RL-under-triage-arxiv.tex
\documentclass{article}
\usepackage{fullpage}
\usepackage[utf8]{inputenc} % allow utf-8 input
\usepackage[T1]{fontenc}    % use 8-bit T1 fonts
\usepackage{hyperref}       % hyperlinks
\usepackage{url}            % simple URL typesetting
\usepackage{booktabs}       % professional-quality tables
\usepackage{amsfonts}       % blackboard math symbols
\usepackage{nicefrac}       % compact symbols for 1/2, etc.
\usepackage{microtype}      % microtypography
\usepackage{xcolor}         % colors
\usepackage{authblk}
\usepackage{amsthm}
\usepackage{mathrsfs}
\usepackage{mathtools}
\usepackage{amsmath}
\usepackage{amssymb}
\usepackage{cancel}
\usepackage{breqn}
\usepackage{enumitem}
\usepackage[numbers]{natbib} 
\usepackage[ruled]{algorithm2e} % For algorithms

\SetAlFnt{\small}
\SetAlCapFnt{\small}
\SetAlCapNameFnt{\small}
\SetAlCapHSkip{0pt}
\IncMargin{-\parindent}

\usepackage{tikz}
\usepackage{Definitions}
\usepackage{algorithmic}
\newcommand{\given}{{\,|\,}}
\newcommand{\xhdr}[1]{\vspace{1mm} \noindent{{\bf #1.}}}

\title{Reinforcement Learning Under Algorithmic Triage}
\date{}
\author{{Eleni Straitouri$^{1}$,}
	{Adish Singla$^{1}$,}
	{Vahid Balazadeh Meresht$^{2}$,}	  
	{Manuel Gomez Rodriguez$^{1}$}}
\affil{$^{1}$Max Planck Institute for Software Systems, \{estraitouri, adishs, manuelgr\}@mpi-sws.org
	$^2$University of Toronto, balazadehvahid@gmail.com }

\begin{document}

	\maketitle
	
	\begin{abstract}
		\input{000abstract}
	\end{abstract}
	
	\section{Introduction}
	\label{sec:introduction}
	\input{010introduction}
	abbrvnat
	\section{Learning Under Triage as Learning Options}
	\label{sec:preliminaries}
	\input{020preliminaries}
	
	\section{Overview of Two-stage Actor-Critic Method}
	\label{sec:method}
	\input{030method}
	
	\section{Offline Off-policy Training Using Human Data}
	%\section{State 1: Offline Training Using Human Data}
	\label{sec:off-policy}
	\input{040off-policy}
	
	\section{On-policy Training Using Machine and Human Data}
	%\section{Stage 2: Online Training}
	\label{sec:on-policy}
	\input{050on-policy}
	
	\section{Experiments}
	\label{sec:synthetic}
	\input{060synthetic}
	\section{Conclusions}
	\label{sec:conclusions}
	\input{080conclusions}
	
	\bibliographystyle{plainnat}	 
	\bibliography{refs}
	
	%%%%%%%%%%%%%%%%%%%%%%%%%%%%%%%%%%%%%%%%%%%%%%%%%%%%%%%%%%%%
	\onecolumn
	\clearpage
	\newpage
	
	\appendix
	\input{090appendix}

	\label{sec:appendix}
\end{document}

%% file: 000abstract.tex
Methods to learn under algorithmic triage have predominantly focused on supervised learning settings 
where each decision, or prediction, is independent of each other.
Under algorithmic triage, a supervised learning model predicts a fraction of the instances and humans
predict the remaining ones.
In this work, we take a first step towards developing reinforcement learning models that are 
optimized to operate under algorithmic triage.
To this end, we look at the problem through the framework of options and develop a two-stage 
actor-critic method to learn reinforcement learning models under triage. 
The first stage performs offline, off-policy training using human data gathered in an environment 
where the human has operated on their own. 
The second stage performs on-policy training to account for the impact that switching may have on 
the human policy, which may be difficult to anticipate from the above human data.
Extensive simulation experiments in a synthetic car driving task show that the machine models and the triage policies trained using our two-stage method effectively complement human policies and outperform those provided by several competitive baselines.

%with a car driving simulator environment

%
%Simulation experiments using both synthetic as well as real driving data show that the machine models and the triage policies trained using our two-stage method effectively complement human policies and outperform those provided by several competitive baselines.

%% file: 010introduction.tex
% learning under algorithmic triage
Learning under algorithmic triage is a new learning paradigm which seeks the development of ma\-chine learning models that operate under different automation levels---models that take decisions for a given fraction of instances and leave the remaining ones to humans~\citep{okati2021differentiable, raghu2019algorithmic}.
This new paradigm has also been referred to as learning under human assistance~\citep{de2020regression, de2020classification}, learning to complement humans~\citep{wilder2020learning, bansal2020optimizing}, and learning to defer to an expert~\citep{sontag2020}.
In learning under algorithmic triage, one does not only has to find a machine learning model but also a triage policy which determines who decides, the model or the human, and when.

% promise fullfilled on supervised learning, but nothing on sequential decision 
% making except for Meresht et al.
Existing works have shown early success at fulfilling the promise of algorithmic triage---by working together, 
they have shown that humans and machine learning models achieve a considerably better performance than each of 
them would achieve on their own.
However, they have predominantly focused on supervised learning settings where each decision, or prediction, is independent of each other. A very recent notable exception is the work by~\citet{meresht2021learning}, which learns to switch con\-trol between machine and human agents in a reinforcement learning setting where decisions are dependent. 
However, in contrast to our work, the policies of the machine agents are (pre-)trained to operate under full automation.
While their problem setting is different, a natural extension of their algorithm to our setting achieves lower performance than ours.

% manuel: add the citations below if we find a place where they fit naturally
% deep reinforcement learning with opponent modeling to robustly switch between multiple machine 
% policies~\citep{everett2018learning,zheng2018deep}

% goal and challenges
In this paper, our goal is to develop reinforcement learning models that are op\-ti\-mized to operate under algorithmic triage.
Similarly as in supervised learning, one of the main challenges is that, for each potential triage policy, there is an optimal machine agent, however, the triage policy is also something one seeks to optimize.
Moreover, in comparison with supervised learning, we face two additional challenges.
First, due to safety concerns, the vast majority of reinforcement learning models are trained using simulator environments~\cite{dosovitskiy2017carla, talpaert2019exploring, wymann2000torcs}. 
Unfortunately, in these environments, human data is typically very limited and, as a result, an accurate estimation of human policies is challenging~\cite{dafoe2021cooperative,dafoe2020open, mcilroy2020aligning, scott2006crawdad, kurin2017atari}.
%
% This is in contrast with real environments, including gaming, where historical 
% human data is in\-crea\-sing\-ly available~\cite{mcilroy2020aligning, scott2006crawdad, kurin2017atari}.
%\adish{TBD -- the reason for stating this contrast is unclear.}
%\manuel{I was also doubting about this. The point is to justify the experiments 
% with real human traces in Atari (or taxi). There is no real human traces in 
% simulation environments. If we remove "In contrast", then it is unclear how we 
% overcome the "Unfortunately, in these environments..". Either way, feel free to 
% rewrite or remove.}
%
Second, the presence of switching introduces an additional cognitive load on the sequential 
decision making process~\cite{brookhuis2001behavioural}. As a result, human policies may worsen 
in ways that are difficult to anticipate from historical human data.

\xhdr{Our approach}
We first introduce a modeling framework to learn reinforcement learning models under triage, which
builds upon the framework of options~\cite{SUTTON1999181, Precup2002, Bacon2017, jain2021safe}.
In the framework of options, an option policy determines which intra-option policy picks actions 
until termination, as dictated by a termination policy, at which point the procedure is repeated.
In our modeling framework, the machine and human policies are the intra-option policies, the 
triage policy is the option policy, and the termination policy does not need to be explicitly 
defined because options are interrupting~\cite{SUTTON1999181}, \ie, whenever the current 
intra-option policy is no longer the best choice, termination occurs.
Here, note that, in contrast with the original framework of options, one of the intra-option 
policies---the human policy---is fixed. 

Building on the above modeling framework, we introduce a two-stage actor-critic method to train the triage policy and the machine policy. 
In the first stage, we perform offline, off-policy training of the machine and triage policies using human data gathered 
in an environment where the human has operated on their own. 
In the second stage, we perform on-policy training to optimize the machine and triage policies trained in the first stage.
Our goal is to benefit from historical human data and guarantee safety in the first stage and to account for the impact that switching may have on the human policy in the second stage. 
%
% as well as the broad range of human abilities in the second stage.
%
Finally, we perform a variety of simulation experiments in
% both 
a synthetic car driving task.
% as well as real driving data.
%
Our results show that the machine models and triage policies trained using our two-stage method effectively 
complement human policies and outperform those provided by several competitive baselines.\footnote{To facilitate research in this area, we will release an open-source implementation of our algorithms with the final version of the paper.}

\xhdr{Further related work}
Our work is also related to the areas of learning to defer and human-machine collaboration.
In learning to defer, the goal is to design classifiers that are able to defer 
decisions~\citep{bartlett2008classification,cortes2016learning,geifman2019selectivenet,geifman2018bias,liu2019deep,ramaswamy2018consistent,thulasidasan2019combating, ziyin2020learning}.
To this end, they learn to defer either by considering the defer action as an additional label value 
or by training an independent classifier to decide about deferred decisions. However, there are no 
human experts who make predictions whenever the classifiers defer them---they just pay a constant 
cost every time they defer predictions. Moreover, the classifiers are trained to predict the labels 
of all samples in the training set as in full automation.
The extensive body of work on human-machine collaboration has predominantly considered settings in 
which the machine and the human interact with each other~\citep{brown2019machine,ghosh2020towards-aamas,grover2018learning, hadfield2016cooperative, haug2018teaching, kamalaruban2019interactive, macindoe2012pomcop,nikolaidis2015efficient, nikolaidis2017mathematical, radanovic2019learning, reddy2018shared, taylor2011integrating,torrey2013teaching,tschiatschek2019learner,walsh2011blending,wilson2018collaborative}.
In this context, our work is more closely connected to a line of work that studies switching behavior and switching costs in the context of human-computer interaction~\citep{sch1,sch2,sch3,sch5,sch4}, which we 
see as complementary.

%% file: 020preliminaries.tex
\looseness-1Let $\Scal$ be the state space, $\Acal$ be the set of actions, $c(s, a)$ be the environment cost of action $a$ at state $s$, and the transition dynamics of the environment are given by $p(s' \given s, a)$.
%assume the transition dynamics of the environment are given by $p(s' \given s, a)$.
%
Then, in reinforcement learning under triage, one needs to find:
\begin{itemize}[leftmargin=0.8cm]
    \item[(i)] a triage policy $\tau(d \given s) : \Scal \times \{0, 1\} \rightarrow [0, 1]$, which determines who takes an action at state $s$ denoted by $d(s)$---the machine ($d(s) = 1)$ or the human ($d(s) = 0$). In the remainder, for simplicity, we will write $\tau(d = 1 \given s) = \tau(s)$.
    \item[(ii)] a machine policy $\pi_{\MM}(a \given s) : \Scal \times \Acal \rightarrow [0, 1]$,     which determines which actions are taken for those states $s$ for which $d(s) = 1$.
\end{itemize}
In the above, the human takes actions according to a policy 
$\pi_{\HH}(a \given s) : \Scal \times \Acal \rightarrow [0, 1]$.
Here, for simplicity, we assume that the human policy satisfies the Markov property, \ie, it depends only on the current state $s$, a common assumption in the machine learning, 
psychology, cognitive science and economics literature. 
While it is possible to convert a non-Markovian human policy into a Markovian one in certain cases 
just by changing the state representation~\cite{daw2014algorithmic}, addressing the problem 
of learning under triage in a semi-Markovian setting is left as a very interesting venue 
for future work.

Then, similarly as in standard reinforcement learning, we look for the triage and machine 
policies that result into the lowest expected cost (or, the highest expected reward).
%(or, alternatively, the highest expected reward). 
%
To this end, we define the value function and the action value function under the triage policy $\tau$ as (see Appendix for details)
%(refer to the Appendix for more details)
% ~\ref{app:value-functions}
%\vspace{-1mm}
\begin{equation} \label{eq:value}
%\begin{split}
    v^{\tau}(s) =
\bar{c}_c(\tau(s)) + \sum_{a \in \mathcal{A}}\left[ \tau(s) \pi_{\MM}(a \given  s) \right. \\
 \left. +(1-\tau(s))\pi_{\HH}(a\given s) \right] \sum_{s' \in \Scal}p(s'\given s,a) \left[c(s,a) +v^{\tau}(s')\right]
%\end{split}
\end{equation}
and 
\begin{equation} \label{eq:q}
q^{\tau}(s,a) =\sum_{s' \in \Scal}p(s'\given s,a)\left[ c(s,a) + v^{\tau}(s')\right],
\end{equation}
with $\bar{c}_c(\tau(s)) = \tau(s) c_c(1) + (1-\tau(s)) c_c(0)$, where $c_c(d)$ is the cost of giving control to the human ($d=0$) or the machine ($d=1$) at state $s$, and $c(s, a)$ is the environment cost. Note that we consider an undiscounted reward setting and assume that termination occurs surely under any policy for any initial state~\cite{Puterman1994}.
% , see Section~\ref{sec:off-policy}.

To design our two-stage actor-critic method, it will also be useful to look at the 
problem from the perspective of the options framework~\cite{SUTTON1999181}.
Under this framework, an option policy determines which intra-option policy picks 
actions until termination, as dictated by a termination policy, at which point the 
procedure is repeated.
In our setting, the machine policy $\pi_{\MM}$ and the human policy $\pi_{\HH}$ 
are the intra-option policies, the triage policy $\tau$ is the option policy, and 
the termination policy does not need to be explicitly defined because options are considered as interrupting~\cite{SUTTON1999181}, \ie, whenever the current intra-option policy is no longer the best choice, termination occurs.
In this context, the choice of intra-option policies depends on the option value function 
\begin{equation} \label{eq:option-q}
\begin{split}
%Q^\tau(s,\tau(s)) = c_c(\tau(s)) + \sum_{a \in \mathcal{A}} \left[ \tau(s) \pi_{\MM}(a \given  s) + (1-\tau(s))\pi_{\HH}(a\given s) \right] q^{\tau}(s,a),
Q^\tau(s,1) &= c_c(1) + \sum_{a \in \mathcal{A}} \pi_{\MM}(a \given  s) q^{\tau}(s,a)\\
Q^\tau(s,0) &= c_c(0) + \sum_{a \in \mathcal{A}} \pi_{\HH}(a\given s) q^{\tau}(s,a),
\end{split}
\end{equation}
%
%which is essentially the counterpart of the action value function for options.
which is essentially the action value function for options.
%essentially

%\adish{TBD -- Should we add in Section 2 that we consider undiscounted setting with absorbing states?}
%\manuel{That is a good idea, please, add it.}

% manuel: I move this remark to above, when we introduce the human policy
% \xhdr{Remarks} For simplicity, we have assumed the human policy $\pi_{\HH}(a \given s)$ 
% satisfies the Markov property, \ie, it depends only on the current state $s$, a common assumption in
% the machine learning, psychology, cognitive science and economics literature. 
%
% While it is possible to convert a non-Markovian human policy into a Markovian one in certain cases 
% just by changing the state representation~\cite{daw2014algorithmic}, addressing the problem 
% of learning under triage in a semi-Markovian setting is a very interesting, albeit challenging, 
% venue for future work.

%% file: 030method.tex
To train the triage policy and machine policy, we introduce a two-stage actor-critic method:
\begin{itemize}[leftmargin=0.8cm]
    \item[I.] The first stage performs offline, off-policy training using human data gathered in an 
    environment where the human has operated on their own. More formally, it seeks to find the triage 
    and machine policy that minimize the mean of the value function $v^{\tau}$ with respect to the 
    stationary state distribution induced by the human policy $\pi_{\HH}$ operating on its own.

    \item[II.] The second stage performs on-policy training to optimize the machine and triage 
    policies trained in the first stage by minimizing the mean of the value function $v^{\tau}$ with 
    respect to the stationary state distribution induced by the human policy $\pi_{\HH}$ and the 
    machine policy $\pi_{\MM}$ operating together, as dictated by the triage policy $\tau$.
\end{itemize}

To this end, we will parameterize both the machine policy (the actor) and the value function (the critic),
and express the triage policy in terms of the parameterized value function. In the next two sections, we 
provide more details of each stage in turn. All the proofs are provided in the Appendix.

%% file: 040off-policy.tex
In this section, we first discuss the training of the pa\-ra\-me\-te\-rized machine policy under the true value function and then discuss how to approximate the true value function.
Throughout the section, we build upon a recent line of work on offline off-policy training~\cite{Sutton16, imani2019offpolicy, pmlr-v40-Yu15}, which we adapt to our 
specific problem setting. 
Similarly as in this line of work, we assume that the behavioral policy---the 
human policy---satisfies the coverage assumption and, for any initial state, 
termination occurs surely under the target policy---the human and machine policies 
operating together, as dictated by the triage policy.

\xhdr{Actor}
Let $\Mcal(\Theta)$ be a class of parameterized machine policies. Then, our goal 
is to find the parameters $\theta \in \Theta$ that minimize the following objective function:
\begin{equation} \label{eq:loss-mu-offline}
%\begin{split}
    J(\theta) = \EE_{s \sim d_{\pi_{\HH}}} [v^{\tau}(s)] = \sum_{s \in \Scal} d_{\pi_{\HH}}(s)\Big( \bar{c}_c(\tau(s)) \\
     + \sum_{a \in \Acal} \left[\tau(s)\pi_{\MM,\theta}(a\given s) + (1 - \tau(s))\pi_{\HH}(a\given s)\right] q^{\tau}(s,a) \Big),
%\end{split}
\end{equation}
where $d_{\pi_{\HH}}$ denotes the stationary state distribution induced by the human policy $\pi_{\HH}(\cdot \given s)$. % , where $\tau_0(s) = 0$ for all $s \in \Scal$.

To facilitate our analysis, we will initially assume that $\tau$ is fixed and independent of $\theta$ and later on relax this assumption. 
Under this assumption, the gradient of the above objective function is given by the following theorem:
\begin{theorem} \label{thm:off-policy-pg}
The gradient of the function $J(\theta)$ with respect to the parameters $\theta$ is given by:
\begin{equation} \label{eq:gradient-j-off-policy}
    \frac{\partial J}{ \partial \boldsymbol\theta} 
    = \sum_{s \in \Scal} m(s) \sum_{a \in \Acal} \frac{\partial \pi_{\MM,\theta}}{\partial \boldsymbol\theta} q^{\tau}(s,a)
\end{equation}
where $\tau$ is a fixed triage policy independent of $\theta$, the often called emphatic weightings $\mb = [m(s)]_{s \in \Scal}$ are given by $\mb = (\Ib - \Pb^T)^{-1}\Db \, \db_{\pi_{\HH}}$ with 
$\db_{\pi_{\HH}} = [d_{\pi_{\HH}}(s)]_{s \in \Scal}$,
$\Db = \diag(\tau)$ is a diagonal matrix with entries based on $\tau(s)$, and 
$\Pb = [P(s,s')]_{s, s' \in \Scal}$ where
\begin{equation*}
P(s,s') = \sum_{a \in \Acal}\left( \tau(s)\pi_{\MM,\theta}(a\given s) \right. \\
\left. + (1 - \tau(s))\pi_{\HH}(a\given s) \right)p(s' \given s, a).
\end{equation*}
%
%\adish{Is $\Db$ a $S \times S$ matrix with entries $1$ or $0$ based on $\tau(s)$?}
%\manuel{Yes, feel free to clarify.}
% assuming that $(\Ib - \Pb_{\tau})^{-1}$ exists.
\end{theorem}
However, in practice, to apply the above theorem, we need an estimate of both the emphatic weightings and the product of the gradients 
of the machine action policy and the action value function from a set of recorded human trajectories $\Dcal = \{\Tcal\}$ with $\Tcal = \{(s_t,a_t)\}$. 
For the former, the following proposition equips us with a sequential update rule for an emphatic weighting estimator 
$M_t$ with desirable properties:
%
% In this first training phase we assume an offline setting, in which we have access to a number of trajectories $\tau = 
% \{(s_t,a_t,s_{t+1})\}_{t=1}^L$, that where gathered while the actual human agent was interacting alone with the true 
% environment. So when we refer to time step $t$ we refer to these trajectories. 
% Following the derivations as in \cite{Sutton16}, we have the following: emphasis update
%
\begin{proposition} \label{prop:emphasis}multline
Let $M_t \leftarrow d(s_t) + \varrho_{t-1} M_{t-1}$, where $\varrho_t = \frac{\varpi(a_t\given s_t)}{\pi_{\HH}(a_t\given s_t)}$ and \begin{equation*}
    \varpi(a_t\given s_t) = d(s_t)\pi_{\MM,\theta}(a_t\given s_t) + (1-d(s_t))\pi_{\HH}(a_t\given s_t).
\end{equation*}
Then, it holds that $m(s) = d_{\pi_{\HH}}(s) \lim_{t \rightarrow \infty} \EE[M_t\given s_t=s]$.
\end{proposition}
For the latter, we resort to a commonly used estimator with sequential update rule $\rho_t \delta_t \nabla_{\theta} \ln \pi_{\MM,\theta}$, where $\delta_t = c(s_t,a_t) + v^{\tau}(s_{t+1}) -  v^{\tau}(s_t)$ and $\rho_t = \frac{\pi_{\MM,\theta}(a_t\given s_t)}{\pi_{\HH}(a_t\given s_t)}$. 
After combining both estimators, we have the following proposition:
%
% Proposition 1 of Imani 18 --> actor update rule 
\begin{proposition} \label{prop:gradient-estimator}
Let $M_t$ be given by Proposition~\ref{prop:emphasis}, $\delta_t = c(s_t,a_t) + v^{\tau}(s_{t+1}) -  v^{\tau}(s_t)$ and $\rho_t = \frac{\pi_{\MM,\theta}(a_t\given s_t)}{\pi_{\HH}(a_t\given s_t)}$. Then, % it holds that:
 \begin{equation*}
     \lim_{t \rightarrow \infty} \EE[M_t\rho_t \delta_t \nabla_{\theta} \ln \pi_{\MM,\theta}] = \sum_{s \in \Scal} m(s) \sum_{a \in \Acal} \frac{\partial \pi_{\MM,\theta}}{\partial \boldsymbol\theta} q^{\tau}(s,a)
 \end{equation*}
 \end{proposition}
Consequently, the above results readily yield the following update rule for the parameters of the 
machine policy:
\begin{equation} \label{eq:actor}
     \theta_{t+1} \leftarrow \theta_{t} - \alpha_t M_t\rho_t \delta_t \nabla_{\theta} \ln \pi_{\MM,\theta}(a_t | s_t) |_{\theta = \theta_t}
\end{equation}
where $\alpha_t$ is the learning rate.
 
% convergence for changing switching
Next, we lift the assumption on the triage policy $\tau$ and let it be an $\epsilon$-greedy policy with respect to the option value function $Q^{\tau}$:
\begin{equation} \label{eq:update-switching}
    \tau_{\theta}(s) := 
    \begin{cases}
    1 - \frac{\epsilon}{2} & \text{if } Q^{\tau}(s, 1) \leq Q^{\tau}(s, 0) \\
    \frac{\epsilon}{2} & \text{otherwise}
    \end{cases}
\end{equation}
While under this definition, $\tau$ depends on $\theta$ because $Q^{\tau}$ implicitly depends on $\theta$, the following proposition shows that gradient of the objective function $J(\theta)$ remains unchanged. 
%
% \adish{TBD -- Should we make it explicit that Q is dependent on $\theta$?}
%
% \manuel{That is an option yes, but in reality $v^{\tau}$ and $q^{\tau}$ also 
% depend on $\theta$, so either we change all to include $\theta$ or we do not, 
% otherwise, it may lead to misunderstanding perhaps. Including $\theta$ may make 
% the notation a bit tedious though.}
%
\begin{theorem} \label{thm:off-policy-pg-epsilon-greedy}
Let $\tau(s)$ be an $\epsilon$-greedy policy with respect to $Q^{\tau}$, then,
$\partial J/\partial \boldsymbol\theta$ is still given by Eq.~\ref{eq:gradient-j-off-policy}.
\end{theorem}
So far, we have assumed that the human policy $\pi_{\HH}(a \given s)$ and the value function $v^{\tau}$ are known. 
However, in practice, we need to estimate their values to be able to compute $\varrho_t$ and $\delta_t$.
To estimate the human policy, we can just use a Montecarlo estimate $\hat{\pi}_{\HH}(a \given s)$ from the recorded human trajectories $\Dcal$, \ie, $\hat{\pi}_{\HH}(a \given s) = \sum_{t} \II(a_t = a, s_t = s)/\II(s_t = s)$.
%
% However, note that this estimator does not account for the impact that the triage % policy may have on the human policy---that is the main reason to perform on-policy training in a % second stage. 
%
To estimate the value function, we will resort to a critic, which we discuss next.

% In practice, there are two issues with this update because both the exact human policy and value function are unknown. Since 
% the actual human policy distribution is unknown, based on the gathered human experience in the trajectories, we approximate 
% the human agent policy distribution $\pi_{\HH}(a\given s)$ with $\hat \pi_{\HH}(a\given s)$, which is the average times action $a$ was 
% chosen in $s$. In practice, we use this approximation for the computation of $\rho_t$, $\varrho_t$ and also for the behaviour 
% policy $\mu$. For the value function that is required for the computation of $\delta_t$, we use the estimates of the critic, 
% which we present in the following section. 

% Manuel: I think this is not necessary and we can safely assume we know the cost given a value of s_t and a_t
% We also assume knowledge of the environment cost function $c'$ and the control 
% cost $c_c$-however if they can be observed, then they can be included in the trajectories and no knowledge of them is 
% required.

\xhdr{Critic}
We adapt the one step update proposed in the options framework~\cite{Bacon2017} to our offline, off-policy setting with emphatic weightings. 
To this end, we start by rewriting the value function in terms of the option value function, \ie, 
\begin{equation*}
    v^{\tau}(s) = \tau(s)Q^{\tau}(s, 1) + (1 - \tau(s))Q^{\tau}(s, 0),
\end{equation*}
and approximating the option value function using a linear function, \ie, $Q^{\tau}_{\boldsymbol \vartheta}(s, d(s)) = \boldsymbol \vartheta^T \boldsymbol \phi(s, d(s)) + c_c(d(s))$, 
where $\varthetab \in \mathbb{R}^n$ is a parameter vector and $\boldsymbol\phi(s, d(s)) \in \mathbb{R}^n$ 
is a feature vector.\footnote{We note that the linear representation is considered for the theoretical results; in the experiments we use neural representations.\label{footnote:linear_representation}}
% \adish{TBD -- Should this Q be with superscript $\tau$? See the 
% change in Section 2.}
% \manuel{I changed both Q and the approximation of Q to include a 
% superscript \tau}
%
Then, given a set of recorded human trajectories $\Dcal$, we use the following update rule, based on
ETD(0)~\cite{Sutton16}, to estimate $\varthetab$:
%the parameters 
%
\begin{equation} \label{eq:vartheta_update}
%\begin{split}
\boldsymbol \vartheta_{t+1}  \leftarrow \boldsymbol \vartheta_t + \beta_t F_t \varrho_t [ C_{t+1} + Q^{\tau}_{\boldsymbol \vartheta_t}(s_{t+1}, d(s_{t+1})) \\ 
- Q^{\tau}_{\boldsymbol \vartheta_t}(s_t, d(s_t)) ] \nabla Q^{\tau}_{\boldsymbol \vartheta}(s_t, d(s_t))|_{\boldsymbol \vartheta = \boldsymbol \vartheta_t}, 
% \nonumber \\   \boldsymbol \vartheta_{t+1}  &\leftarrow \boldsymbol \vartheta_t + \beta_t F_t \varrho_t \left[ C_{t+1} + \tau(s_{t+1})\boldsymbol\vartheta^T_{t}\boldsymbol\phi(s_{t+1}, 1) + (1 - \tau(s_{t+1}))\boldsymbol\vartheta^T_{t}\boldsymbol\phi(s_{t+1}, 0) \right. \nn \\ 
%   & \qquad \qquad \qquad \qquad \qquad \qquad \qquad \left. + \bar{c}_c(\tau(s_{t+1})) -  \boldsymbol\vartheta^T_{t}\boldsymbol\phi(s_t,\tau(s_t)) - \bar{c}_c(\tau(s_t)) \right] \boldsymbol\phi(s_t,\tau(s_t)),
%\end{split}
\end{equation}
 where $C_{t+1}=c(s_t,a_t) + c_c(d(s_t))$, 
$\varrho_t = \frac{\varpi(a_t\given s_t)}{\pi_{\HH}(a_t\given s_t)}$, 
$F_t = i(s_t) + \varrho_{t-1}F_{t-1}$ is the emphatic weighting, with $i(s_t) \in \{0,1\}$,
and $\beta_t$ is the learning rate. 
Here, one may think of setting $i(s_t) = 1$ for states that have been visited by the human operating on their own and $i(s_t) = 0$ otherwise and can estimate the human policy $\pi_{\HH}$ similarly as in the actor training.
Moreover, note that $F_t$ is the counterpart of $M_t$ in the actor training with $i(s_t)$ instead of $d(s_t)$.
%
%  to be an interest function that for example can be set to 1 for those states that are in the human history. 
% $F_t$ is the emphatic weighting, similar to $M_t$ in the previous section - we use $F$ instead of $M$ to avoid confusion.
 
Whenever $\tau$ is fixed and independent of $Q^{\tau}$, we are able to prove almost sure asymptotic convergence of the above update 
rule. 
In practice, we have observed good empirical performance whenever $\tau$ is a $\epsilon$-greedy policy with respect to $Q^{\tau}$. 
%, as shown in Section~\ref{sec:experiments}.
%
Before we proceed with the analysis, we rewrite the second term of the right 
hand side of Eq.~\ref{eq:vartheta_update} as:
\begin{equation*}
    %\boldsymbol \vartheta_{t+1}  &{\leftarrow} \boldsymbol \vartheta_t \\
    % +& 
    \beta_t \left[ \underbrace{F_t \varrho_t( C_{t+1} + c_c(d(s_{t+1})) - c_c(d(s_t)))\boldsymbol\phi(s_{t},d(s_t))}_{\boldsymbol b_t} \right. \\
    - \left. \underbrace{F_t \varrho_t \boldsymbol\phi(s_t,d(s_t))(\boldsymbol\phi(s_t,d(s_t)) - \boldsymbol\phi(s_{t+1},d(s_{t+1})))^T }_{\boldsymbol A_t} \boldsymbol \vartheta_t \right]
\end{equation*}
Following~\citet{pmlr-v40-Yu15}, if we are able to show that the following conditions are satisfied, then we can 
guarantee almost sure asymptotic convergence of the above update rule:
%
% because the multiplicative factor $(I - \beta_t\boldsymbol A_t)$, which for $\beta_t$ small enough is always 
% between 0 and 1. 
%
 \begin{itemize}[leftmargin=0.8cm]
    \item[---] The human policy satisfies the coverage assumption and induces an irreducible Markov chain. % on the state space. 
    \item[---] Termination occurs surely under target policy for any initial state, meaning $(I - \Pb)^{-1}$ exists.
    \item[---] The learning rate sequence $\{\beta_t\}$ satisfies $\beta_t  \in (0,1], \sum_t\beta_t < \infty, \sum_t \beta_t^2 < \infty$, is deterministic and eventually non increasing.
    \item[---] $\Ab = \lim_{t \rightarrow \infty} \EE_{\mu}[\boldsymbol A_t]$ is non singular.
\end{itemize}
The first two conditions are satisfied by assumption, the third condition only requires to set $\beta_t$ accordingly. 
and the fourth condition is satisfied if $\Ab$ is positive definite, a property that we also need for the update rule to be stable~\cite{Sutton16}. 
To prove the latter, we first rewrite $\Ab$ as (refer to the Appendix for details) $\boldsymbol A = \boldsymbol \Phi^T \boldsymbol F (I - \boldsymbol P%_{\tau}
) \boldsymbol \Phi$,
% ~\ref{app:matrix-A}
%
%\begin{align}
%    \boldsymbol A = \boldsymbol \Phi^T \boldsymbol F (I - \boldsymbol P_{\tau}) \boldsymbol 
%\Phi,
% \end{align}
%
where $\boldsymbol \Phi = \Db \boldsymbol\Phi_{1} + (I - \Db)\boldsymbol\Phi_{0}$, with $\boldsymbol \Phi_d$ a $|S|\times n$ matrix with rows $\boldsymbol \phi^T(s,d)$, 
$\Pb, \Db$ are defined in Theorem~\ref{thm:off-policy-pg} 
and $\boldsymbol F$ a diagonal matrix with elements $f(s)$. 
Then, it is easy to see that, to prove that $\Ab$ is positive definite, it is sufficient to prove that 
$\boldsymbol K = \boldsymbol F (I - \boldsymbol P)$ is positive definite. To this end, we have
the following claim:
\begin{claim} \label{claim}
A symmetric matrix $\boldsymbol K$, with positive diagonal and negative off diagonal elements, is positive definite if each row sum plus the corresponding column sum of it is positive
\end{claim}
Then, we only need to compute each row sum plus the corresponding column sum of $\boldsymbol K$ to verify if $\boldsymbol K$ 
is positive definite.
Since $\Pb$ is a stochastic matrix, the row sums of $\boldsymbol K$ will be $0$, so the column sums of 
$\boldsymbol K$ must be positive. 
Now, we can find the column sums of $\boldsymbol K$ if we multiply it with a row vector with all elements equal 
to $1$. 
So the row vector of the column sums of $\boldsymbol K$ are given by:
\begin{equation*}
    1^T\boldsymbol F(I - \Pb) % \boldsymbol f^T(I - \Pb) 
    \ost{(i)}{=}\boldsymbol \db_{\pi_{\HH}}^T \boldsymbol D_{i}  (I - \boldsymbol P)^{-1}(I - \boldsymbol P)
    = \db_{\pi_{\HH}}^T \boldsymbol D_{i},
\end{equation*}
where $\boldsymbol D_{i} = \diag(i(s))$, 
$\mathbf{f} = [f(s)]_{s \in \Scal}$ and, 
in (i), we have used that $\boldsymbol f = (I - \boldsymbol P^T)^{-1}\boldsymbol D_{i}\boldsymbol \mu $. 
As a result, it readily follows that, if $i(s) = 1$ for all $s \in \Scal$, then $\boldsymbol K$ is positive definite. 
Even if that is not the case, if $i(s) \geq 0$ and there are $n$ linearly independent feature vectors of states with $f(s) > 0$, it can be shown that $\boldsymbol K$ is also positive definite~\cite{pmlr-v40-Yu15}.

\vspace{1mm} We summarize the resulting actor-critic method for 
off-policy training in Algorithm~\ref{alg:off-policy}. 

\begin{algorithm}[t]
  \caption{Actor-critic method for off-policy training\label{alg:off-policy}}
  \textbf{Input:} $Q^{\tau}_\vartheta$, $\pi_{\MM, \theta}$, $\pi_{\HH}$, $c$, $c_c$, $\epsilon$, $i$, $\alpha$, $\beta$, set of trajectories $\Dcal$ \\ % = \{ (s_t, a_t) \} \\
%   $\theta_0, \vartheta_0 \sim \Ncal(0,1)$
  \textbf{Initialize:}  $F_0 \gets 0$, $M_0 \gets 0$, $\varrho_{0} \gets 1$, $t \gets 0$, $\theta_0 \sim \Ncal(0,1)$,$\vartheta_0 \sim \Ncal(0,1)$. \\ 
  %, $t \gets 0$. \manuel{Not sure this is needed, it is just an subindex}
  
  % \vspace{1mm}
  % Initialize networks $\theta_0 \sim \Ncal(0,1)$, $\vartheta_0 \sim 
  % \Ncal(0,1)$\\
  % Initialize $F_0 \gets 0, M_0 \gets 0 , \varrho_{0} \gets 0, t \gets 0$\\
  \vspace{1mm}
  \For{\textnormal{trajectory} $\Tcal$ in $\Dcal$}{
%   Sample trajectory  from $\Dcal$ \\
% \vspace{1mm} 
\For{$(s_t, a_t, s_{t+1})$ in $\Tcal$}{
  % c(s_t,a_t) \manuel{I do not think we should include c(s_t, a_t) in \Tcal}
  $t \gets t + 1$\\
  $d_t \gets \epsilon$-greedy($Q^{\tau}_{\vartheta_t}(s_t, 0), Q^{\tau}_{\vartheta_t}(s_t, 1)$)\\
% \vspace{1mm} 
\tcp{critic update}
  $C_{t+1} \gets c(s_t, a_t) + c_c(d_t)$\\
  $F_t \gets i(s_t) + \varrho_{t-1}F_{t-1}$\\
  $\varrho_t \gets \frac{d_t \pi_{\MM}(a_t|s_t) + (1-d_t) \pi_\HH(a_t |s_t)}{\pi_\HH(a_t |s_t)}$\\
  $d_{t+1} \gets \epsilon$-greedy($Q^{\tau}_{\vartheta_t}(s_{t+1}, 0), Q^{\tau}_{\vartheta_t}(s_{t+1}, 1)$)\\
  $\delta^{Q}_t \gets  C_{t+1} + Q^{\tau}_{ \vartheta_t}(s_{t+1}, d_{t+1}) - Q^{\tau}_{\vartheta_t}(s_t, d_t)$ \\
$\vartheta_{t+1}  \gets \vartheta_t + \beta_t F_t \varrho_t \delta^Q_t\nabla_{\vartheta} Q^{\tau}_{\vartheta}(s_t, d_t)|_{\vartheta = \vartheta_t}$\\

% \vspace{1mm} 
\tcp{actor update}
$\delta_t = c(s_t,a_t) + Q^{\tau}_{ \vartheta_t}(s_{t+1}, d_{t+1}) - Q^{\tau}_{\vartheta_t}(s_t, d_t)$ \\
$\rho_t = \frac{\pi_{\MM,\theta}(a_t\given s_t)}{\pi_{\HH}(a_t\given s_t)}$\\
$M_t \gets d_t + \varrho_{t-1}M_{t-1} $\\
$\theta_{t+1} \gets \theta_{t} - \alpha_t M_t\rho_t \delta_t \nabla_{\theta} \ln \pi_{\MM,\theta}(a_t | s_t) |_{\theta = \theta_t}$
}
}
\Return $\theta_t$, $\vartheta_t$
\end{algorithm}
% \manuel{Please, add an algorithm box and briefly explain its structure in text.}

%% file: 050on-policy.tex
% overview
In this section, we proceed similarly as in the previous section. We first discuss the training of the parameterized machine policy under the true value function and then the approximation of the true value function. 

\xhdr{Actor}
Let $\Mcal(\Theta)$ be a class of parameterized machine policies and, given a human policy $\pi_{\HH}$ and machine policy $\pi_{\MM, \theta}$, denote the action policy induced by 
the triage policy $\tau$ as
\begin{align}
\label{def:overall_policy}
    \pi_{\theta}(a \given s) = \tau(s) \pi_{\MM, \theta}(a \given s) + (1 - \tau(s)) \pi_\HH(a \given s)
\end{align}
Then, our goal is 
to find the parameters $\theta \in \Theta$ that minimize the following objective function:
\begin{equation} \label{eq:loss-mu-online}
%\begin{split}
    J(\theta) = \EE_{s \sim d_{\pi_{\theta}}} [v^{\tau}(s)] = \sum_{s \in \Scal} d_{\pi_{\theta}}(s) \Big( \bar{c}_c(\tau(s)) \\
     + \sum_{a \in \Acal} \left[\tau(s)\pi_{\MM,\theta}(a\given s) + (1 - \tau(s))\pi_{\HH}(a\given s)\right] q^{\tau}(s,a) \Big),
%\end{split}
\end{equation}
where $d_{\pi_{\theta}}$ denotes the stationary state distribution induced by the policy $\pi_{\theta}$.

Assume $\tau$ is fixed and independent of $\theta$. Then, we have the following theorem, which readily follows from the standard policy gradient theorem~\cite{sutton2018reinforcement}:
\begin{theorem} \label{thm:on-policy-pg}
The gradient of the function $J(\theta)$ with respect to the parameters $\theta$ is given by:
\begin{equation} \label{eq:gradient-j-on-policy}
    \frac{\partial J}{ \partial \boldsymbol\theta} 
    = \EE_{s \sim d_{\pi_{\theta}}}\left[\tau(s) v^{\tau}(s) \frac{\nabla_\theta \pi_{\MM,\theta}(a \given s)}{\pi_{\MM,\theta}(a \given s)}\right]
\end{equation}
\end{theorem}
The above result yields the following update rule for the parameters of the machine policy:
\begin{align}
    \label{eq:gradient-update-on-policy}
    \theta_{t + 1} \leftarrow \theta_t - \alpha_t d(s_t) v^{\tau}(s_t) \frac{\nabla_\theta \pi_{\MM,\theta}(a_t\given s_t)|_{\theta = \theta_t}}{\pi_{\MM,\theta}(a_t\given s_t)} ,
\end{align}
where $\alpha_t$ is the learning rate. Here, note that the parameters of the machine policy are only updated whenever the triage policy let the machine take action ($d(s_t) = 1$).
Moreover, if $\alpha_t$ is chosen properly, one can use standard arguments to show that the above update rule converges. 

Now, assume the triage policy $\tau$ is chosen $\epsilon$-greedily with respect to the option value function, \ie, $\tau$ is given by Eq.~\ref{eq:update-switching}. 
Then, it can be shown similarly as in the previous section that the gradient of $v^{\tau}$ remains the same and thus Theorem~\ref{thm:on-policy-pg} still holds and we can still use 
the update rule given by Eq.~\ref{eq:gradient-update-on-policy} to find the parameters of the machine policy. 
%
% \adish{Does this reasoning from Eq. (7) still work here given that $\pi_{\HH}$ 
% depends on $\tau$ in this equation?}
%
% \manuel{We had to remove the explicit dependence on \tau since we could not fix 
% the proof of Theorems 2 in off-policy and the statement here.}

In the above, we have assumed that the value function $v^{\tau}$ is known. However, in practice, we will estimate their value using a critic, which we discuss next.

\looseness-1\xhdr{Critic}
% TD(0)
Here, we use the same linear approximation for the option value function as in the previous section, \ie, $\hat Q^{\tau}(s, d(s)) = \boldsymbol \vartheta^T \boldsymbol \phi(s, d(s)) + c_c(d(s))$ (also, see Footnote~\ref{footnote:linear_representation}). 
Then, we use the following update rule, based on TD(0)~\cite{sutton2018reinforcement}, to estimate the parameters $\varthetab$, \ie,
\begin{equation*}
%\begin{split}
    \vartheta_{t+1} \leftarrow \vartheta_{t} - \beta_t [ C_{t+1} + Q^{\tau}_{\boldsymbol \vartheta_t}(s_{t+1}, d(s_{t+1})) \\
    - Q^{\tau}_{\boldsymbol \vartheta_t}(s_t, d(s_t)) ] \nabla Q^{\tau}_{\boldsymbol \vartheta}(s_t, d(s_t))|_{\boldsymbol \vartheta = \boldsymbol \vartheta_t},
%\end{split}
\end{equation*}
which can be shown to converge using standard arguments whenever $\tau$ is fixed and independent of $Q$.

\vspace{1mm} We summarize the resulting actor-critic method for on-policy 
training in Algorithm~\ref{alg:on-policy}. 
\begin{algorithm}[t!]
  \caption{Actor-critic method for on-policy training\label{alg:on-policy}}
  \textbf{Input:} $Q^{\tau}_\vartheta$, $\pi_{\MM, \theta}$, $c$, $c_c$, $\epsilon$, $\alpha$, $\beta$, \#episodes. \\
%  $Q^{\tau}$: option value function parameterized by $\vartheta$\\
%  $\pi_\MM$: machine policy parameterized by $\theta$\\
%  $\epsilon$: parameter of $\epsilon$-greedy switching policy
%  i: an interest function\\
%  $\alpha, \beta$: actor and critic learning rates\\
%   \textbf{Initialize:} $\theta_0, \vartheta_0 \sim \Ncal(0,1)$
  %, $t \gets 0$. \manuel{Not sure this is needed, it is just an subindex}
  
  % \vspace{1mm}
  % Initialize networks $\theta_0 \sim \Ncal(0,1)$, $\vartheta_0 \sim 
  % \Ncal(0,1)$\\
   \textbf{Initialize} $\theta_0 \sim \Ncal(0,1)$,$\vartheta_0 \sim \Ncal(0,1)$, $t \gets 0$\\
  
  \vspace{1mm}
  \For{$j \in 1,2,...,$\#\textnormal{episodes}}{
  % each episode }{
%\vspace{1mm}  
\While{episode is not terminated}{
  % c(s_t,a_t) \manuel{I do not think we should include c(s_t, a_t) in \Tcal}
  $t \gets t + 1$\\
  $d_t \gets \epsilon$-greedy($Q^{\tau}_{\vartheta_t}(s_t, 0), Q^{\tau}_{\vartheta_t}(s_t, 1)$)\\
  Sample $a_t$ based on $d_t$ and $s_t$; get next state $s_{t+1}$\\
% \vspace{1mm} 
\tcp{critic update}
  $C_{t+1} \gets c(s_t, a_t) + c_c(d_t)$\\

  $d_{t+1} \gets \epsilon$-greedy($Q^{\tau}_{\vartheta_t}(s_{t+1}, 0), Q^{\tau}_{\vartheta_t}(s_{t+1}, 1)$)\\
  
  $\delta^{Q}_t \gets  C_{t+1} + Q^{\tau}_{ \vartheta_t}(s_{t+1}, d_{t+1}) - Q^{\tau}_{\vartheta_t}(s_t, d_t)$ \\
$\vartheta_{t+1}  \gets \vartheta_t + \beta_t  \delta^Q_t\nabla_{\vartheta} Q^{\tau}_{\vartheta}(s_t, d_t)|_{ \vartheta = \vartheta_t}$\\

% \vspace{1mm} 
\tcp{actor update}
$v^{\tau}_{\vartheta_t} \gets d_t Q_{\theta_t}^{\tau}(s_t, 1) + (1 - d_t) Q_{\theta_t}^{\tau}(s_t, 0) $
% $\delta_t = c(s_t,a_t) + Q^{\tau}_{ \vartheta_t}(s_{t+1}, d_{t+1}) - Q^{\tau}_{\vartheta_t}(s_t, d_t)$ \\

$\theta_{t+1} \gets \theta_{t} - \alpha_t d_t v^{\tau}_{\vartheta_t} \nabla_{\theta} \ln \pi_{\MM,\theta}(a_t | s_t) |_{\theta = \theta_t}$

  }
  
  }
  \Return $\theta_t$, $\vartheta_t$
\end{algorithm}

%\manuel{Please, add an algorithm box and briefly explain its structure in text.}

%% file: 060synthetic.tex
% In this section we demonstrate the setup and the results from simulation experiments using synthetic human data.  Comparing the performance of our algorithm to non cooperative baselines, we aim at illustrating the effectiveness of learning to complement a human policy instead of learning to outperform it.
%
The goal of our experiments is to demonstrate that our two-stage actor-critic method
is able to identify the limitations and complementary strengths of a given human policy 
and a machine policy from a given pa\-ra\-me\-te\-rized class of models. 
% and it eventually finds a machine model within 
% the class that, when operating under triage, achieves a superior performance than the 
% \emph{best} machine within the class trained to operate on its own.
%
To this end, we focus on three research questions (\textit{RQs}):
\begin{itemize}
    \item[--] \looseness-1\textit{RQ1}: Within a given parameterized class, can our method find a machine model that when operating under triage achieves better performance than the human policy or a machine policy trained to operate on its own? % within the class 
%    the joint training of the machine model and triage result in an effective combination of the 
% complementary strengths of human and machine, to achieve a joint performance higher than individual?
    \item [--]\textit{RQ2}: In scenarios where humans wish to keep agency, can our method 
    find triage policies that \emph{only} give control to the machine policy to avoid perilous situations?
    \item[--] \textit{RQ3}: In scenarios where the human policy changes due to switching, how much competitive advantage does our two-stage training process bring by adapting to these changes? % , in comparison with a triage method using a fixed machine policy?
    %can our method effectively learn \emph{how} 
    %switching changes the policy of the human so that the performance under triage does not 
    %degrade due to this change?
\end{itemize}

%\looseness-1
%Next, we describe the experimental setup and discuss different methods for comparison in order to answer the above RQs.
%we applied, the scenarios we examined as well as our results. 
%

\subsection{Environment for synthetic car driving task}
%
% To answer the above RQs, 
We consider a synthetic driving task and a environment based on previous work~\cite{levine2010feature,kamalaruban2019interactive,meresht2021learning}.
%
%inspired by a car driving simulator 

\xhdr{Environment design, episodes, and objective} 
%Our environment design is based on the work of  \cite{meresht2021learning} and we consider different scenarios to answer specific \emph{RQs} (see next subsection).
Based on the environment used in \cite{meresht2021learning}, we begin by generating a grid-based task, consisting of three lanes with infinite rows. Each row $r$ is characterized by a traffic level $\gamma_{r} \in \{\texttt{no-car}, \texttt{light}, \texttt{heavy}\}$, based on which, the cell types of $r \in \{\texttt{road}, \texttt{grass}, \texttt{stone}, \texttt{car}\}$ are sampled independently at random.
% ; the cell types can be one of $\{\texttt{road}, \texttt{grass}, \texttt{stone}, \texttt{car}\}$. 
%
Refer to the Appendix for more details on the sampling of traffic levels and cell types. 
% on the sampling the traffic level of each row and the probability of each cell type given the row's 
% traffic level.
% are provided in the appendix of the supplementary material.
%depends on the traffic level of the previous row and is sampled randomly according to the transition diagram given in the appendix.
%shown in Figure~\ref{fig:transition-diagram}.
%Table \ref{tab:type_distribution}
%
%- Episode discussion here. Remove this from Action space
% \xhdr{Episodes and objective}
To train the policies, we consider an episodic setting in which, at the beginning of every episode, the driving agent starts interacting with the environment from the middle lane of a randomly chosen row and terminates after a finite horizon of $20$ steps.\footnote{The presence of the agent in a cell of type \texttt{car} does not indicate episode termination, it just symbolizes proximity to another car, enough to be in danger, but not an accident.} The overall objective of an agent is to minimize the cost of the trajectory induced by the sequence of the actions taken during each episode.

\xhdr{State space}
As a state representation we consider the current cell type (i.e., the type of the cell in which the agent is at the current time step), followed by the cell types of the next six rows in front.
%, from left to right
% - Padding idea, and number of bits in the final state space representation.
For the state feature vector representation, we use one-hot encoding of four bits (one bit for each distinct  cell type), so the resulting state feature vector is of size $4 \times (1 + 6 \times 3)$, \ie, $76$ bits. Once there are less than six rows in the horizon, zero-padding is used. 
As an agent has to learn a generalized policy that takes actions for any input state, we do not consider tabular representations and instead use neural policies for the agents as discussed below.
%We caution the reader to not trivialize this setup -- with a small grid-world learning and we use neural policies that take the 

%If the current cell belongs to the leftmost (rightmost) lane, only the reachable cells of the immediate next row are included in the state, meaning that next row's cell types are shifted to the right (left) and padded left (right) with an extra ghost cell of type \texttt{\textbf{wall}}. 
%
\xhdr{Action space}
We assume that the chosen driving agent---which is the only one in motion---moves always forward at each time step $t$ and decides whether to go left, straight or right, taking one of the actions in $\Acal =\{\texttt{left}, \texttt{middle}, \texttt{right}\}$ respectively. If the agent is either in the leftmost or rightmost lane, it will never choose an action that leads out of the grid, \ie, never chooses \texttt{left} (\texttt{right}) when in leftmost (rightmost) lane.  

\looseness-1\looseness-1\xhdr{Cost functions}
To define the environment cost $c(s,a)$, we first associate a fixed cost related to each cell type indicating its negative impact on the driving agent. We set these cell type costs as follows: \texttt{grass} as $2$, \texttt{stone} as $4$, \texttt{car} as $10$,  and \texttt{road} as $0$. Given a state action pair $(s,a)$, we define the environment cost $c(s,a)$ as the cell type cost associated with the next cell when an action $a$ is taken from $s$. For the control cost $c_c$, we vary it across different evaluation scenarios.
%is set according to different scenarios that we consider for evaluation.
% is the cost associated with the driving (either the human or machine agent) 
%is set  accordingly, depending on the scenario for each agent (machine or human). 
%(impact cost)
%of  \texttt{cell_x}let \texttt{cell_x} be the next current cell if action $a$ is taken from $s$ then 

% At every time step we assume two types of cost that contribute to the trajectory cost, the control cost and the environment cost. The first reflects the cost of picking a specific agent (e.g. cost for using the machine agent), whereas the second represents the cost of moving to a cell with an obstacle.
% As a result, the environment cost depends on the cell type of the next current cell, based on the action that is taken, with every cell type  having a different fixed cost related to its negative impact. Cells of type   \texttt{road} have $0$ cost and cells of type \texttt{grass}, \texttt{stone}, \texttt{car} have costs of $2$, $4$, and $10$ respectively. 
\begin{figure*}[t!]
%TODO opt color --> light red / marker every
    %\captionsetup[subfigure]{font=scriptsize,labelfont=scriptsize}
    \centering
    \begin{subfigure}[t]{0.3\textwidth}
        \includegraphics[width=\textwidth]{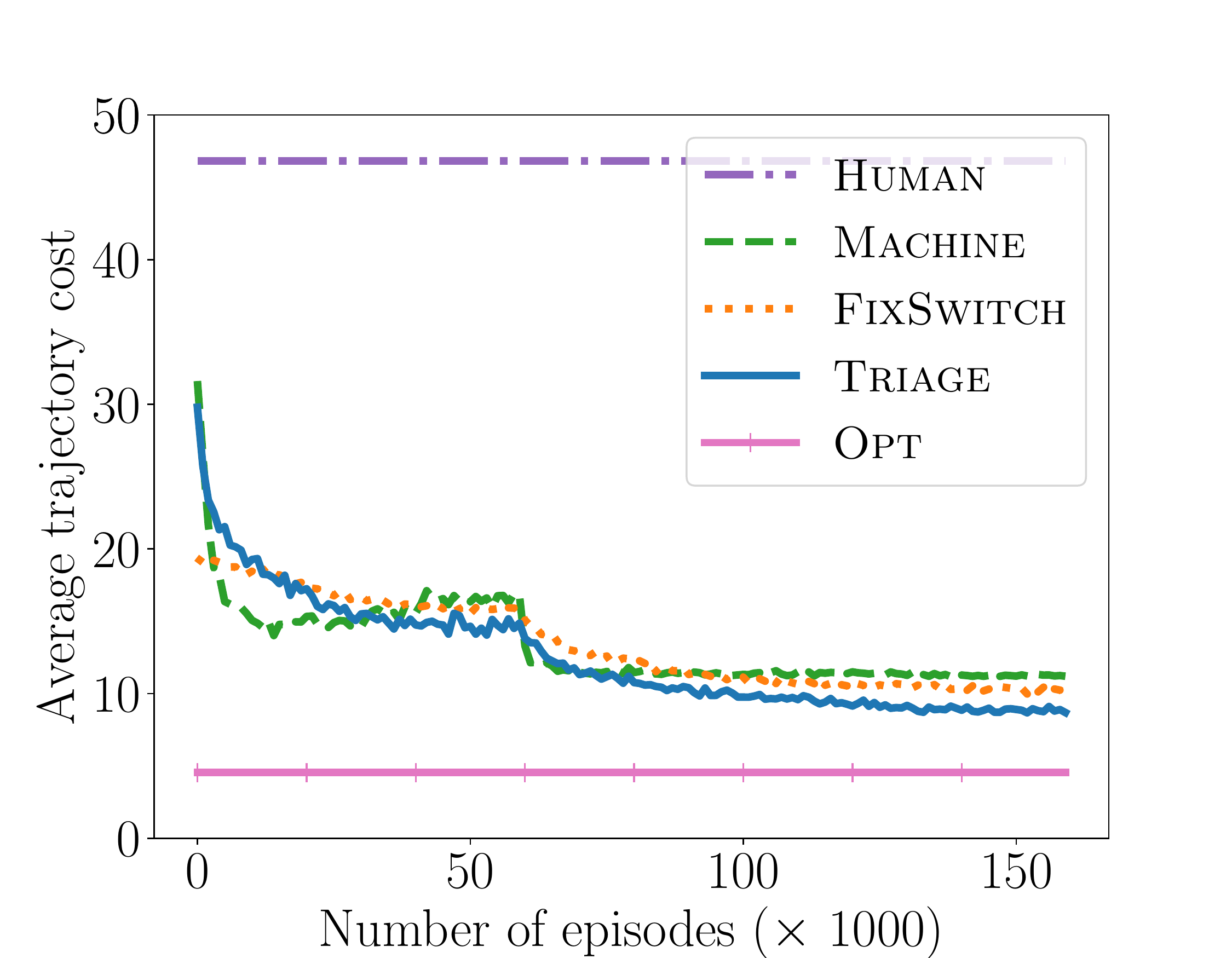}
        \caption{\textit{Scenario I}}
        \label{subfig:quantitative-a}
    \end{subfigure}
    \quad
    \begin{subfigure}[t]{0.3\textwidth}
        \includegraphics[width=\textwidth]{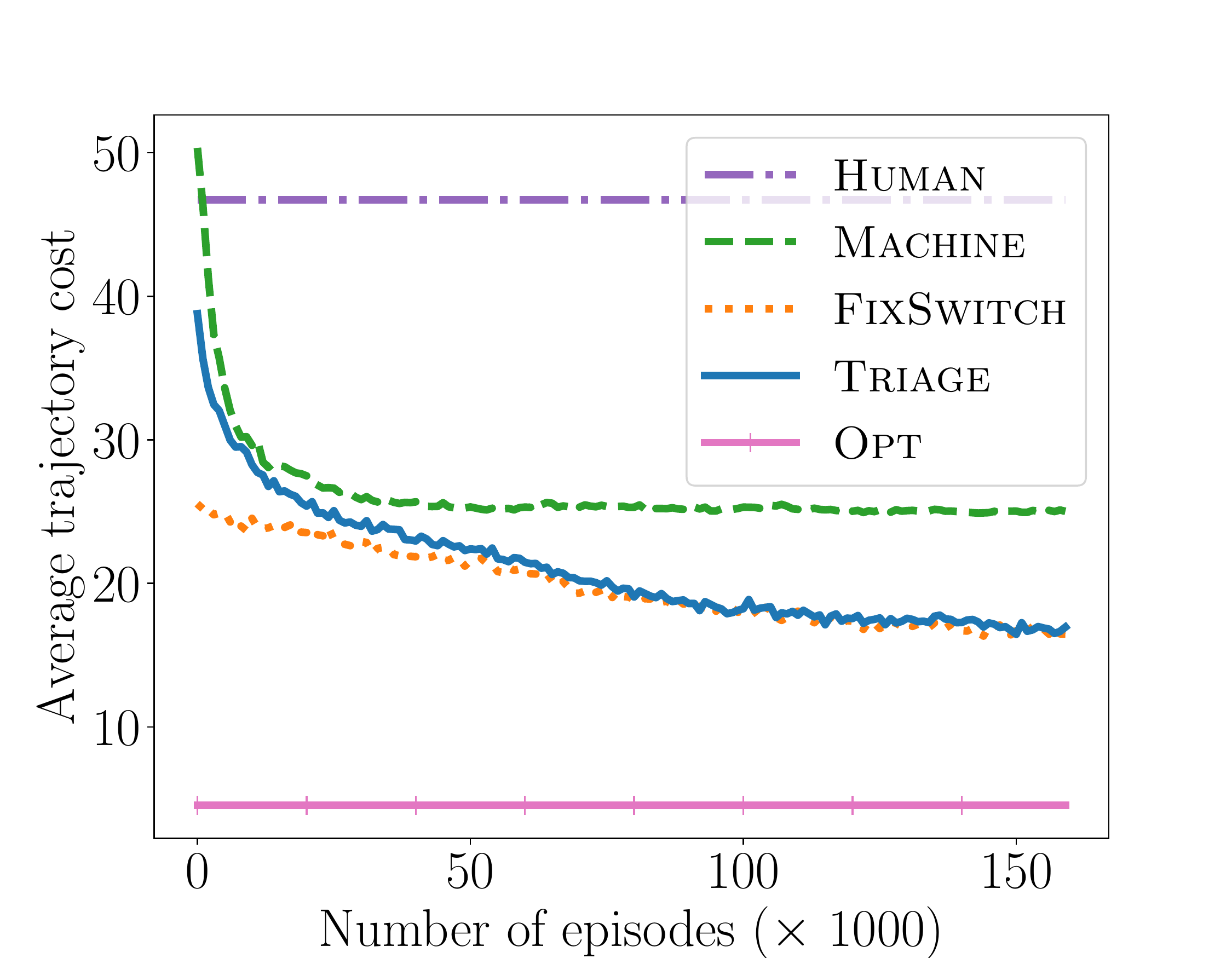}
        \caption{\textit{Scenario II}}
        \label{subfig:quantitative-b}
    \end{subfigure}
    \quad
    \begin{subfigure}[t]{0.3\textwidth}
        \includegraphics[width=\textwidth]{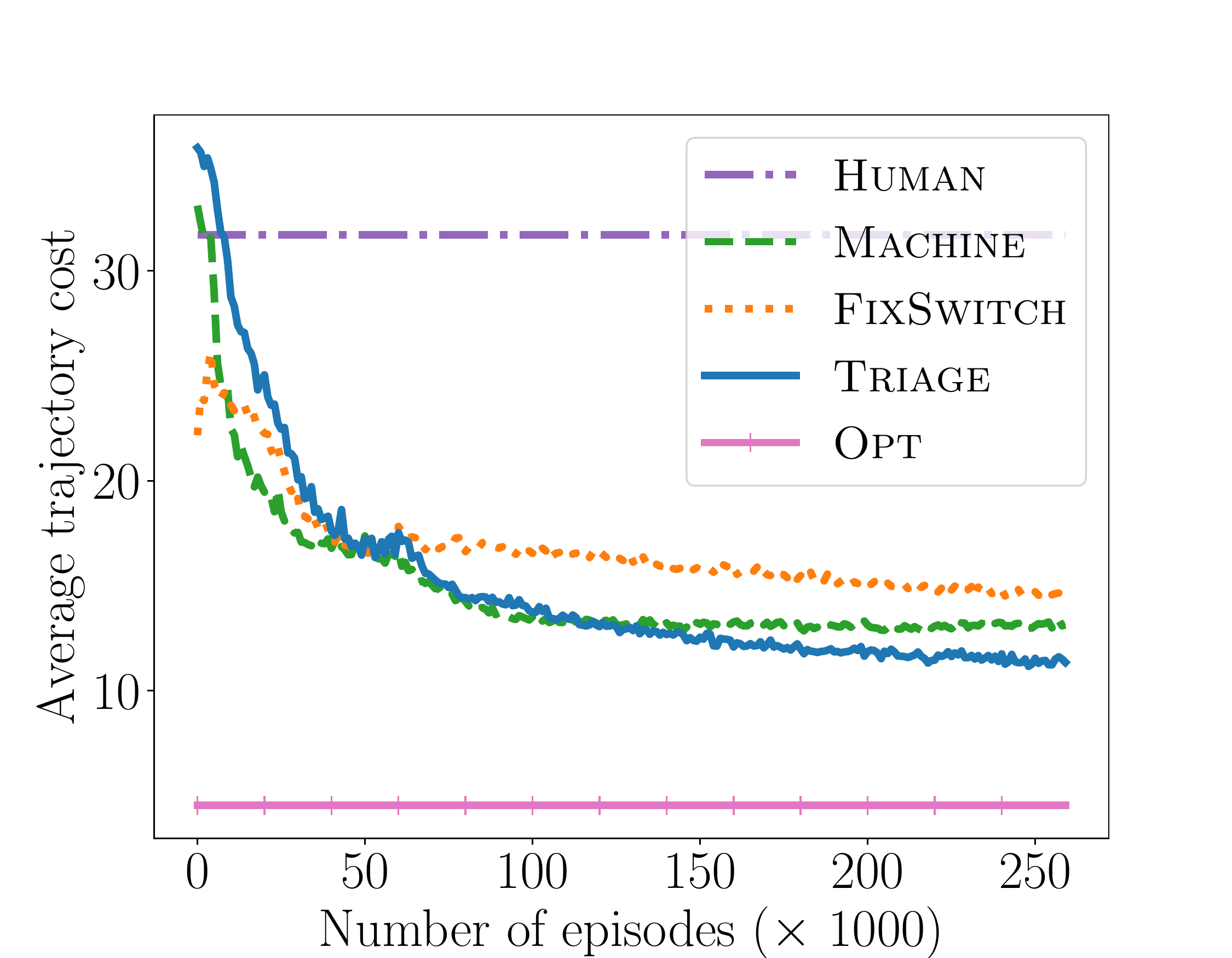}
        \caption{\textit{Scenario III}}
        \label{subfig:quantitative-c}
    \end{subfigure}
    \caption{Test average trajectory cost against the number of training episodes. For offline training, all methods use $60$K episodes. For online training, methods in \textit{Scenario I} and \textit{Scenario II} use $100$K episodes; methods in \textit{Scenario III} use $200$K episodes. See the main text for discussion of these results.
    %in these plots 
    %\textbf{(a)}, \textbf{(b)}, and \textbf{(c)}  show results for  \textit{Scenario I}, \textit{Scenario II}, and \textit{Scenario III} respectively. 
    %(\textbf{(c)}) episodes. 
    % There was offline training in the first $60$K episodes in all  with $60$K offline training episodes and $100$K online training episodes.  \textbf(b) shows  results with $60$K offline training episodes and $100$K online training episodes.   shows  results with $60$K offline training episodes and $200$K online training episodes. 
    %for different scenarios and five methods;
    %\algtriage achieved better or similar performance to all baselines in every \textit{Scenario}. Transitioning from the off-policy to the on-policy training phase has caused a small jump for \algmachine due to slight overfitting in \textit{I} and not complete convergence in \textit{III}.
    }
    \label{fig:quantitative-plots}    
    %\caption{Average trajectory cost on test set sampled every 1000 episodes for each \textit{Scenario}.  \algtriage achieved better or similar performance to all baselines in every \textit{Scenario}. 
    %Transitioning from the off-policy to the on-policy training phase has caused a small jump for \algmachine due to slight overfitting in \textit{I} and not complete convergence in \textit{III}.}    
\end{figure*}

%%%%%%%%%%%%%%%%%%%%%%%%%%%%%%%%%%%%%%%%
% \begin{figure*}[ht]

%   \subfloat[\textit{Scenario I }]{
% 	\label{subfig: quantitative-a}
% 	\begin{minipage}[c][0.23\linewidth]{
% 	   1\textwidth}
% 	   \centering
% 	   \includegraphics[width=1.3\textwidth]{plots/synthetic1_costs.pdf}
% 	\end{minipage}}
%  \hfill 	
%   \subfloat[\textit{Scenario II}]{
%     \label{subfig: quantitative-b}
% 	\begin{minipage}[c][.9\width]{
% 	   0.3\textwidth}
% 	   \centering
% 	   \includegraphics[width=1.3\textwidth]{plots/synthetic2_costs.pdf}
% 	\end{minipage}}
%  \hfill	
%   \subfloat[\textit{Scenario III} ]{
%     \label{subfig: quantitative-c}
% 	\begin{minipage}[c][0.9\width]{
% 	   0.3\textwidth}
% 	   \centering
% 	   \includegraphics[width=1.3\textwidth]{plots/synthetic3_costs.pdf}
% 	\end{minipage}}
% \caption{ Average trajectory cost on test set sampled every 1000 episodes for each \textit{Scenario}.  \algtriage achieved better or similar performance to all baselines in every \textit{Scenario}. 
% Transitioning from the off-policy to the on-policy training phase has caused a small jump for \algmachine due to slight overfitting in \textit{I} and not complete convergence in \textit{III}.
% }
% \label{fig:quantitative-plots}
% \end{figure*}

\subsection{Methods and scenarios}
%We evaluate a total of five different methods including three baselines (\alghuman, \algmachine, and \algfxd), our algorithm (\algtriage), and an optimal plan given a specific episode (\algopt). In order to answer each of the aforementioned \textit{RQs}, we evaluate all methods in three distinct scenarios, one for each \textit{RQ}. Next we describe these methods and scenarios in detail. 
\looseness-1We evaluate a total of five different methods (\alghuman, \algtriage, \algmachine, \algfxd, \algopt). These methods are evaluated in three distinct scenarios (\textit{I}, \textit{II}, \textit{III}), one for each \textit{RQ}. Method \algopt represents an optimal planning given a specific episode, providing an upper bound on the performance. Below, we describe the remaining methods and scenarios. We have provided the code and further implementation details as part of the supplementary material, including the neural network architectures used for the actor and the critic in different methods.\footnote{We ran all experiments on a Debian machine equipped with Intel\textsuperscript{\textregistered} Xeon\textsuperscript{\textregistered} E7-8857v2 CPU @ 3.00GHz and 16GB memory.}
%\algtriage, \algmachine, \algfxd).
%In all the above methods, we use neural networks to implement both the actor % (machine policy)
%and the critic. Refer to the Appendix for details. %

%\footnote{We have provided the code and further implementation details as part of the supplementary material; we will release an open-source implementation with the final version of the paper.}
%Explain --
%to facilitate research in this area, 
%of our algorithms
% in detail

\looseness-1\xhdr{Method \alghuman}
This method corresponds to the human agent acting alone. In our implementation, we assume a myopic human policy which chooses an action that 
 minimizes the one-step environment cost (with ties being broken randomly).
% moving the car in the cell with the minimum cost. 
% In case there are more than one greedy options, the agent randomly picks one among them. 
Apart from these myopic choices, the human policy is also suboptimal because of ignorance regarding a scenario-dependent type of obstacle that is mistakenly perceived as \texttt{road}. 
%, as discussed below

\xhdr{Method \algtriage} This is our two-stage actor-critic method. We consider an $\epsilon$-greedy triage policy with respect to the option value function,
that chooses between a human policy (as described in \alghuman) and a trainable
machine policy (from a parametric class of models as specified per scenario).
%that, depending on the scenario, may also always ignore a specific type of obstacle as it is ignored in \alghuman. 
%
%
%In particular, we use two different neural networks 
%The machine policy network has as input the perceived state feature vector and as output 
%The input of the option value function network comprises of the state feature vector 
%
%
%
For the first offline training stage using human data, we use trajectories of the human policy operating alone for several episodes ($60$K) and apply Algorithm \ref{alg:off-policy} for training actor and critic networks. In this stage, there is no interaction with the true environment and the training uses the recorded transitions from the human trajectories; for the computation of the emphatic weightings, we assume access to the true human policy distribution.  For the second online training stage, both the actor and critic networks continue their training using Algorithm \ref{alg:on-policy} while being deployed in the true environment. For the number of training episodes per stage, see Figure \ref{fig:quantitative-plots}.
We use the following schedule for $\epsilon$: (a) in the first half of the offline stage we set $\epsilon=0.2$, (b) for the second half of the offline stage we set $\epsilon=0.1$, and (c) for the entire online stage we decay $\epsilon$ from $0.1$ with rate $1/\sqrt{t}$.
% ($t$ is per thousand episodes)

% for 60K episodes for scenarios \textit{I} and \textit{II} and for 50K for \textit{III} and assumed access to the true 
% human policy distribution. In the first half of this stage, for the triage policy, we set  $\epsilon=0.2$  and for the second half $\epsilon=0.1$. The 
% following on-policy training stage lasted 100K episodes for \textit{I} and \textit{II} and 200K for \textit{III}. In this stage we decayed $\epsilon$
% from $0.1$ with rate $1/\sqrt{t}$ ($t$ is per thousand episodes)%
%

\xhdr{Method \algmachine} This method has the same two-stage actor-critic framework as in \algtriage, with the following crucial difference: instead of an $\epsilon$-greedy triage policy, we assume a triage policy that always chooses the machine policy.
%and never gives control to human.
%We use the same network architectures and training schedules as in \algtriage.

\looseness-1\xhdr{Method \algfxd} This method also uses the same training framework as in \algtriage, except that instead of a trainable machine policy, the $\epsilon$-greedy triage policy chooses between the human policy and a pretrained fixed machine policy. This pretrained machine policy (\ie, a fixed actor) corresponds to the machine policy at the end of the offline training stage in \algmachine. 

%The same network and training schedule as in \algtriage is used for the option value function learning.

%\xhdr{Neural network representations} 
%%%%%%%%%%%%
%In all the above methods, we use neural networks to implement both the actor % (machine policy)
%and the critic. Refer to the Appendix for details. 
%%%%%%%%%%%%
% (computing option value function).
%
%In our implementation of above methods, we use neural networks for both the machine policy (actor) and option value function (critic). 
% The machine neural policy has the following representation: (a) the input is $76$ binary features corresponding to the perceived state, (b) the output is the estimated log probability of each action, and (c) there is $1$ hidden layer with $256$ units and hyperbolic tangent activation. The option value function network has the following representation: (a) the input is $78$ binary features ($76$ for the state and additional $2$-bits encoding the agent in control), (b) the output is the estimated option value, and (c) there is $1$ hidden layer same as for the machine neural policy.
%\algtriage, \algmachine, and \algfxd
 
% - Let's explain in this order: \alghuman, \algtriage, \algmachine, \algfxd, \algopt
% - For machine / fixed, it might be easier to discuss how they differ from Triage. We need to spend more time discussing Fixed and Machine -- I don't think Fixed is easy to understand.

\xhdr{Scenarios} We design the following three scenarios, one for each \textit{RQ}, to systematically evaluate different methods: 
\begin{itemize}
    \item[--]\looseness-1\textit{Scenario I} 
    % captures the suboptimality of human and machine policies when operating alone
    exemplifies a situation where the performance of the human and the machine is suboptimal when operating alone: the human policy ignores cells of type \texttt{car} and any machine policy in the class ignores cells of type \texttt{grass}. We set $c_c(0) = c_c(1) = 0$. 
    
    \looseness-1\item[--]\textit{Scenario II} 
    % captures that the human wishes to keep agency
    exemplifies a situation where the human wishes to keep agency. To this end, we set $c_c(0) = 0$,  $c_c(1) = 1$. Here, the human policy ignores cells of type \texttt{car} whereas any machine policy in the class perfectly recognizes all cell types. 
    
    \item[--]\textit{Scenario III} 
    %captures the changes in human's policy in the presence of switching. 
   exemplifies a situation where the human policy changes in the presence of switching. We consider a human policy which always ignores cells of type \texttt{grass}; furthermore, the human policy momentarily ignores \texttt{car} cell type at the time step of switching (as the human might not be attentive at this time). In this scenario, any machine policy ignores \texttt{stone}, and we set $c_c(0) = 1$, $c_c(1) = 0$. 

    %if there was switching of control from the machine agent to the human agent, the human agent   
    % The motivation for this scenario is to examine, whether \algtriage can account for a changes in human behaviour because of switching, specifically loss of attention due to regain of control.
\end{itemize}

% For a quantitative comparison, ... steps) and present the results in Figure~\ref{fig:quantitative-plots}
%     

%     
%       
%     
% \

% Add cluster characteristics ?
\subsection{Results}
%We compare different methods using both the quantitative as well as qualitative results (see Figure~\ref{fig:quantitative-plots} and Figure~\ref{fig:qualitative-plots}). Next
In this section, we discuss results for three scenarios in the context of different \textit{RQs}. For a quantitative comparison, we compute the average trajectory cost for the methods on a separate test set that corresponds to $1000$ randomly generated episodes (each with horizon of $20$ steps).
%In particular, we compute this cost for different number of training episodes 
% For the qualitative results we present 
%We present  in Figure \ref{fig:quantitative-plots} the average trajectory cost on a fixed test set of size 1000, sampled at every 1000 episodes of training for our algorithm and all baselines.
% In Figure~\ref{fig:quantitative-plots}, we provide convergence plots by showing the average trajectory cost for different number of training episodes on the x-axis. 
Figure~\ref{fig:quantitative-plots} summarizes the results for different number of training episodes. Here, for evaluation in the offline stage, we set $\epsilon = 0$ when computing the performance on the test set; in the online stage, we use the same $\epsilon$ as the one used in the corresponding training episode. For a qualitative comparison, we provide illustrative examples of trajectories induced by the examined methods at the end of the training in Figure~\ref{fig:qualitative-plots}. 
% Comments on results

\xhdr{Results for \textit{Scenario I} (\textit{RQ1})} Recall that for the \textit{Scenario I}, the human policy is suboptimal because it chooses actions myopically and ignores \texttt{car} cell types, and any machine policy in the considered parametric class is suboptimal because it ignores \texttt{grass} cell types. Figure~\ref{subfig:quantitative-a} shows that our method \algtriage achieves better performance than \alghuman or \algmachine. Moreover, Figure~\ref{subfig:qualitative-a} highlights how our method \algtriage is able to identify the limitations and complementary strengths of the human and the machine agent, and selectively gives control to each of them.  
%Figure~\ref{subfig:quantitative-a}  shows that our triage method \algtriage outperforms both \algmachine and \alghuman. 
%The effective combination of each agent's strengths becomes obvious in Figure \ref{subfig:qualitative-a} where the triage policy gave control at the agent, that could recognize better each state. Given that the parametric class of the machine ignores \texttt{grass}, 

%Within a given parameterized class, can our method find a machine model that, when operating under triage, achieves better performance than the human or a machine within the class trained to operate on its own achieve? 

\xhdr{Results for \textit{Scenario II} (\textit{RQ2})} Recall that for the \textit{Scenario II}, the human agent wishes to keep agency and this is captured by control costs $c_c(0) = 0$,  $c_c(1) = 1$. 
Similar to \textit{Scenario I}, Figure~\ref{subfig:quantitative-b} shows that \algtriage achieves a higher performance than both \alghuman and \algmachine.
Moreover, for this scenario, \algtriage is able to find triage policies that only give control to the machine policy to avoid perilous situations (see  Figure \ref{subfig:qualitative-b} for an illustrative example). Quantitatively, we note that the \algtriage method gives about only $25\%$ of control to the machine policy per episode in \textit{Scenario II}, in comparison to about $45\%$ in \textit{Scenario I}.
%Furthermore, we note that the \algtriage gives an average control given to the machine agent per episode corresponds to about $25\%$ in \textit{Scenario II}, in comparison to about $45\%$ \textit{Scenario I}.
%In Figure \ref{subfig:quantitative-b} \algtriage has by far outperformed both the agents acting individually. As shown in Figure \ref{subfig:qualitative-b}, this is achieved by giving control to human most of the times and switching to machine only to prevent from actions that may lead to \texttt{car} cells.

\begin{figure}[t!]
  %\subfloat[\algtriage  vs \algopt, \textit{Scenario I}]
  %\subfloat[\textit{Scenario I}]{
    \begin{subfigure}[t]{0.32\textwidth}
	    \begin{minipage}[c][1.2\width]{1\textwidth}
	    \centering
	        \includegraphics[width=0.31\textwidth]{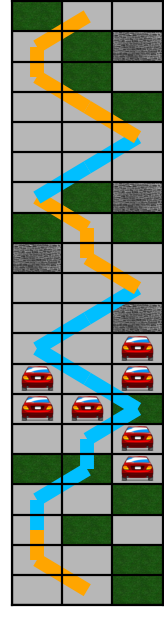}
	        \includegraphics[width=0.31\textwidth]{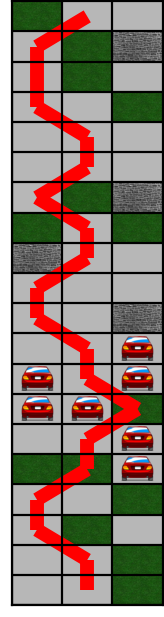}
	   \end{minipage}
	   \caption{\textit{Scenario I}} %: \algtriage  vs \algopt 
	   \label{subfig:qualitative-a}	   
    \end{subfigure}
    \hfill 	
    %\subfloat[\algtriage  vs \alghuman, \texc8312df50atit{Scenario II}] 
    \begin{subfigure}[t]{0.32\textwidth}
	    \begin{minipage}[c][1.2\width]{1\textwidth}
	    \centering
	        \includegraphics[width=0.31\textwidth]{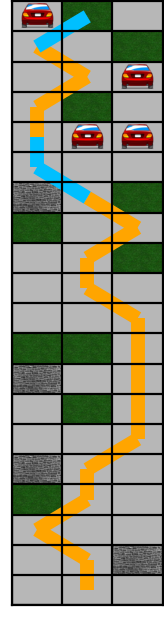}
	        \includegraphics[width=0.31\textwidth]{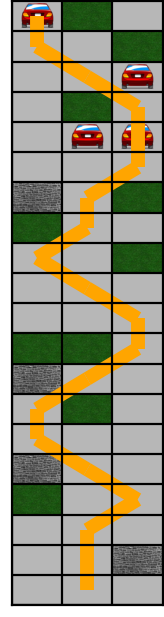}
        \end{minipage}
        \caption{\textit{Scenario II}} %: \algtriage  vs \alghuman) 
        \label{subfig:qualitative-b}        
    \end{subfigure}	
    \hfill	
    %\subfloat[\algtriage  vs \algfxd, \textit{Scenario III}]
    \begin{subfigure}[t]{0.32\textwidth}
	    \begin{minipage}[c][1.2\width]{1\textwidth}
	    \centering
	        \includegraphics[width=0.31\textwidth]{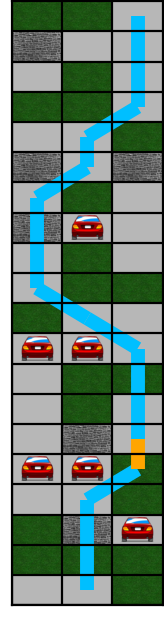}
	        \includegraphics[width=0.31\textwidth]{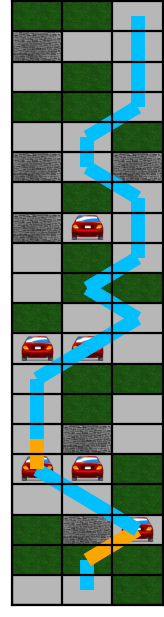}
	   \end{minipage}
	   \caption{\textit{Scenario III}}	%: \algtriage  vs \algfxd   
	   \label{subfig:qualitative-c}
   \end{subfigure}	
   \caption{Illustrative trajectories for different scenarios, where orange (blue) indicates human (machine) is in control and red indicates the optimal plan. 
   In each panel, the left trajectory always corresponds to \algtriage and the right trajectory corresponds to (a) \algopt, (b) \alghuman and (c) \algfxd.
%   \textbf{(a)} compares \algtriage (left) vs. \algopt (right) for \textit{Scenario I}. 
% \textbf{(b)} compares \algtriage (left) vs. \alghuman (right) for \textit{Scenario II}. 
% \textbf{(c)} compares \algtriage (left) vs. \algfxd (right)  for \textit{Scenario III}. 
   %See the main text for discussion of these results.
   }
   \label{fig:qualitative-plots}   
   %\caption{Qualitative compressresults showing example trajectories for different scenarios --  orange indicates control of human, blue control of machine and red is used for the optimal plan. \textbf{(a)} compares \algtriage (left) vs. \algopt (right) for \textit{Scenario I}; we observe an effective combination of the agents, achieving a trajectory with minimum cost. \textbf{(b)} compares \algtriage (left) vs. \alghuman (right) for \textit{Scenario II}; we can observe  meaningful interventions that prevented human from danger of crashing. \textbf{(c)} compares \algtriage (left) vs. \algfxd (right)  for \textit{Scenario III}; we can see that in the 3rd row from below, \texttt{car} was avoided due to lack of switching, given human's ignorance immediately after regaining control.}
\end{figure}

\looseness-1\xhdr{Results for \textit{Scenario III} (\textit{RQ3})} Recall that for the \textit{Scenario III}, there is a change in human behavior in the presence of switching. 
As a result, while in the previous two scenarios, the \algfxd method and the \algtriage method perform comparably,
in this scenario, \algtriage exhibits a competitive advantage
w.r.t. \algfxd, as shown in Figure~\ref{subfig:quantitative-c}.
This is because the \algfxd method uses a pretrained machine policy (a fixed actor) and therefore is unable to adapt to changes in the human policy, in contrast with the \algtriage method. 
%
%Figure \ref{subfig:qualitative-c} (2nd row) exemplifies the (lack of) adaptation of the \algtriage method (\algfxd method). 
Figure \ref{subfig:qualitative-c} (2nd row) exemplifies the adaptation of the \algtriage method, in contrast to that of \algfxd. 
Here, \algfxd gives control to the human as the machine policy ignores \texttt{stone} (see 3rd row, middle lane); however, this leads to an undesirable outcome as the human policy would ignore \texttt{car} momentarily after switching (see 3rd row, right lane). As can be seen, the \algtriage method has adapted to these changes and keeps the control with the machine policy in the 2nd row. The results in \textit{Scenario III} highlight the importance of the second stage in our method, which enables an adaptation to unforeseen behavioral changes of the human.  

%% file: 080conclusions.tex
In this paper, we have initiated the development of reinforcement learning models that are optimized
to operate under algorithmic triage.
We have formalized the problem building upon the framework of options and introduced a two-stage actor-critic 
method to train both the triage policy and the policy of the reinforcement learning agent.
%
% Moreover, we have evaluated our algorithm using both synthetic and real human data and have shown that 
% the triage policies and reinforcement learning agents trained using our method outperform several 
% competitive baselines.
%
Our work opens many interesting venues for future work. For example, it would be interesting to derive 
convergence guarantees for the update rule used in the critic both during offline off-policy training and 
on-policy training whenever the triage policy depends on the option value function.
%
% To this end, it may be useful to use the two-time scale approach by~\citet{pmlr-v119-zhang20s}.
%
% manuel: i remove this because otherwise we would have to really explain those arguments
% and similar arguments as those by Konda et al.~\cite{konda2002}.
%
In that context, an analysis of convergence for the whole actor-critic algorithm would be a breakthrough result.
Moreover, in our theoretical results, we have assumed that the estimation of the human policy is perfect, however, 
it would be interesting to account for error estimations in the analysis.
%
%
%
% Regarding the off-policy phase, we present the following directions of further analysis of the proposed algorithm:
% \begin{enumerate}
%    \item According to Zhang et al.\cite{pmlr-v119-zhang20s} the computation of $M_t$ can cause large approximation error and a possible solution may
% be to use function approximation as in \cite{pmlr-v119-zhang20s}
%     \item Convergence of the critic given a non fixed overall action policy (maybe use a two-timescale approach as in \cite{pmlr-v119-zhang20s} and 
% arguments from \cite{konda2002})
%    \item Convergence of the whole algorithm
%    \item Guarantees on the approximation error due to human policy estimation
%\end{enumerate}
%
% For the on-policy algorithm, also a more formal analysis of convergence can be a direction for future work.
%
Finally, it would be valuable to assess the performance of reinforcement learning models that are optimized to operate under
algorithmic triage using interventional experiments on a real-world application.

%% file: 090appendix.tex
\section{Derivation of Bellman's equations} \label{app:value-functions}
Let $C_{t+1} = c_c(d(s_{t})) + c(s_{t}, a_{t})$ be the total cost at time step $t$ and assume that the terminal states are absorbing, \ie, the only transition in those states are to themselves with zero cost. Then, the value function given the switching policy $\tau$ is given by:
\begin{align*}
    v^{\tau}(s) &= \EE \left[\sum_{k=0}^{\infty} C_{t+k+1}\given s_t \text{$=$} s\right] = \EE \left[ c_c(d(s_t)) + c(s_t, a_t) + \sum_{k=0}^{\infty} C_{t+k+2} \given s_t \text{$=$} s\right] \\ 
    &= 
    \bar{c}_c(\tau(s)) + \sum_{a \in \mathcal{A}}\left( \tau(s)\cdot \pi_{\MM}(a \given  s) + (1-\tau(s))\pi_{\HH}(a\given s) \right) \\
    & \qquad \qquad \qquad \qquad \qquad \qquad \times \sum_{s' \in \mathcal{S}}p(s'\given s,a) \left[c(s,a) + \EE_{\tau}\left[\sum_{k=0}^{\infty} C_{t+k+2} \given s_{t+1} \text{$=$} s' \right]\right] \\
    &= \bar{c}_c(\tau(s)) + \sum_{a \in \mathcal{A}}\left( \tau(s) \cdot \pi_{\MM}(a \given  s) + (1-\tau(s))\pi_{\HH}(a\given s) \right)\sum_{s' \in \mathcal{S}}p(s'\given s,a) \left[c(s,a) +v^{\tau}(s')\right]
\end{align*}
Moreover, the action value function given the switching policy $\tau$ is given by:
\begin{align*}
q^{\tau}(s,a) &= \EE_{\tau} \left[\sum_{k=0}^{\infty} C_{t+k+1}\given s_t \text{$=$} s,a_t \text{$=$} a\right]
= \EE_{\tau} \left[ c(s_t,a_t) + \sum_{k=0}^{\infty} C_{t+k+2}\given s_t \text{$=$} s,a_t \text{$=$} a\right] \\
&= \sum_{s' \in \mathcal{S}}p(s'\given s,a)\left[ c(s,a) + \EE_{ \pi} \left[\sum_{k=0}^{\infty} C_{t+k+2}\given s_{t+1} \text{$=$} s'\right]\right]
=\sum_{s' \in \mathcal{S}}p(s'\given s,a)\left[ c(s,a) + v^{\tau}(s')\right]
\end{align*}

\section{Proof of Theorem~\ref{thm:off-policy-pg}} \label{app:off-policy-pg}
By definition, we have that:
\begin{equation} \label{eq:nabla_J}
    \frac{\partial J}{ \partial \theta} = 
    \sum_{s \in \Scal} d_{\pi_{\HH}}(s) \frac{\partial v^{\tau}(s)}{\partial \theta}
\end{equation}
Now, to compute the gradient of the value function $v^{\tau}(s)$, we apply the chain rule:
\begin{equation} \label{eq:nabla_v}
\frac{\partial v^{\tau}(s)}{\partial \theta} = 
 \tau(s) \sum_{a \in \mathcal{A}} \frac{\partial \pi_{\MM,\theta}}{\partial \theta} q^{\tau}(s,a)
 + \sum_{a \in \mathcal{A}} \left( \tau(s)\pi_{\MM,\theta}(a\given s) 
+ (1 - \tau(s))\pi_{\HH}(a\given s) \right)
\sum_{s' \in \mathcal{S}}p(s'\given s,a)\frac{\partial v^{\tau}(s')}{ \partial \theta}
\end{equation}
Next, let $\db_{\pi_{\HH}} = [d_{\pi_{\HH}}(s)]_{s \in \Scal}$,
$\dot {\boldsymbol{v}}^{\tau} = \left[\frac{\partial v^{\tau}(s)}{\partial \theta_{s'}}\right]_{s, s' \in \Scal}$,
$\Gb = [G(s,s')]_{s, s' \in \Scal}$ with
\begin{equation*}
    G(s,s') = \sum_{a \in \mathcal{A}} \frac{\partial \pi_{\MM,\theta}(a\given s)}{\partial \theta_{s'}}q^{\tau}(s,a),
\end{equation*}
$\Pb = [P(s,s')]_{s, s' \in \Scal}$ with
\begin{equation*}
P(s,s') = \sum_{a \in \Acal}\left( \tau(s)\pi_{\MM,\theta}(a\given s) 
+ (1 - \tau(s))\pi_{\HH}(a\given s) \right)p(s' \given s, a),
\end{equation*}
and $\Db= \diag(\tau)$. 
Then, we can rewrite Eq.~\ref{eq:nabla_v} as:
\begin{equation*} % \label{eq:nabla_v-hat-matrix}
\dot{\boldsymbol{v}}^{\tau} = \Db \Gb + 
\boldsymbol P \dot{\boldsymbol{v}}^{\tau} \Rightarrow \dot{\boldsymbol{v}}^{\tau}{=}  (\boldsymbol I - \boldsymbol P )^{-1} \Db \Gb
\end{equation*}
Therefore, we have that:
\begin{equation*} \label{eq:nabla_j-proof}
\frac{\partial J}{ \partial \theta} = \db_{\pi_{\HH}}^T \dot{\boldsymbol{v}}^{\tau} = \db_{\pi_{\HH}}^T (I - \Pb)^{-1}\Db \Gb = \boldsymbol m^T \Gb = \sum_{s \in \Scal} m(s) \sum_{a \in \Acal} \frac{\partial \pi_{\MM,\theta}}{\partial \theta} q^{\tau}(s,a)
\end{equation*}

\section{Proof of Proposition~\ref{prop:emphasis}} \label{app:Mt-update}
Let $\bar{m}(s) := d_{\pi_{\HH}}(s)\lim_{t \rightarrow \infty} \EE[M_t\given  {s_t=s}]$. Then, 
we have that:
\begin{align*}
\bar{m}(s) &= d_{\pi_{\HH}}(s)\lim_{t \rightarrow \infty} \EE[d(s_t) +  \varrho_{t-1} M_{t-1} \given  {s_t = s}] \\
     &\ost{(i)}{=} d_{\pi_{\HH}}(s)\tau(s) + d_{\pi_{\HH}}(s)\lim_{t \rightarrow \infty} \sum_{s', a'}\PP\{{s_{t-1} = s'}, {a_{t-1}=a'} \given  {s_t=s}\}\frac{\varpi(a'\given s')}{\pi_{\HH}(a'\given s')} \EE[M_{t-1} \given  {s_{t-1}=s'}]\\
     &\ost{(ii)}{=} d_{\pi_{\HH}}(s)\tau(s) + d_{\pi_{\HH}}(s)\sum_{s', a'}\frac{d_{\pi_{\HH}}(s')\pi_{\HH}(a'\given  s')p(s\given  s', a')}{d_{\pi_{\HH}}(s)} \frac{\varpi(a'\given s')}{\pi_{\HH}(a'\given s')}\lim_{t \rightarrow \infty}\EE[M_{t-1} \given  {s_{t-1}=s'}] \\
     &= d_{\pi_{\HH}}(s)\tau(s) + \sum_{s'} \sum_{a'}\varpi(a'\given s')p(s\given  s', a')d_{\pi_{\HH}}(s')\lim_{t \rightarrow \infty}\EE[M_{t-1} \given  {s_{t-1}=s'}]\\
     &= d_{\pi_{\HH}}(s)\tau(s) +\sum_{s'} P(s',s) \bar{m}(s'),
\end{align*}
where, in $(i)$, we have used that $\tau(s)=\lim_{t \rightarrow \infty} \EE[d(s_t) \given  s_t=s]$ since $\tau(s)$ is fixed and, 
in $(ii)$, we have used the Bayes rule.
Now, let $\bar{\mb} = [\bar{m}(s)]_{s \in \Scal}$. Then, we can rewrite the above expression as:
\begin{equation*}
\bar{\mb}= \Db \db_{\pi_{\HH}} + \Pb^T \bar{\mb} \quad \Rightarrow \quad 
\bar{\mb} = (I - \Pb^T)^{-1}\Db \, \db_{\pi_{\HH}}
\end{equation*}
Therefore, by definition, it readily follows that $\bar{\mb} = [m(s)]_{s \in \Scal}$.

\section{Proof of Proposition~\ref{prop:gradient-estimator}} \label{app:gradient-estimator}
It readily follows that:
\begin{align*}
 \EE[M_t\rho_t \delta_t \nabla_{\theta}\ln \pi_{\MM,\theta}] &= \sum_{s} d_{\pi_{\HH}}(s) \EE [M_t \given  s_t {=} s] \EE[\rho_t \delta_t \nabla_{\theta}\ln \pi_{\MM,\theta} \given  s_t{=}s] \\
 &= \sum_{s} m(s) \sum_{a} \pi_{\HH}(a\given s) \frac{\pi_{\MM, \theta}(a\given s)}{\pi_{\HH}(a\given s)}\frac{1}{\pi_{\MM,\theta}(a\given s)} \\
 & \qquad \qquad \qquad \qquad \times \frac{\partial \pi_{\MM,\theta}(a\given s)}{\partial \theta} \left( c(s,a) + \sum_{s'}p(s'\given s,a)v^{\tau}(s') - v^{\tau}(s) \right)\\
 &\ost{(i)}{=}  
  \sum_{s} m(s) \sum_{a} \frac{\partial \pi_{\MM,\theta}(a\given s)}{\partial \theta}
  \left( \sum_{s'}p(s'\given s,a)(c(s,a) + v^{\tau}(s')) - v^{\tau}(s) \right)\\
  &\ost{(ii)}{=}
  \sum_{s} m(s) \sum_{a} \frac{\partial \pi_{\MM,\theta}(a\given s)}{\partial \theta}(q^{\tau}(s,a) - v^{\tau}(s)) \\
  & =
  \sum_{s} m(s)  \left(\sum_{a} \frac{\partial \pi_{\MM,\theta}(a\given s)}{\partial \theta}q^{\tau}(s,a)  - v^{\tau}(s) \frac{\partial \sum_{a} \pi_{\MM,\theta}(a,s)}{\partial \theta} \right) \\
  &=
  \sum_{s} m(s)  \left(\sum_{a} \frac{\partial \pi_{\MM,\theta}(a\given s)}{\partial \theta}q^{\tau}(s,a)  - v^{\tau}(s) \frac{\partial 1}{\partial \theta} \right) \\
  &= \sum_{s} m(s)\sum_{a} \frac{\partial \pi_{\MM,\theta}(a\given s)}{\partial \theta}q^{\tau}(s,a),
 \end{align*}
where, in $(i)$, we have used $1\cdot c(s,a) = \sum_{s'}p(s'\given s,a)c(s,a)$ and, in (ii), we have used Eq.~\ref{eq:q}.

\section{Proof of Theorem~\ref{thm:off-policy-pg-epsilon-greedy}} \label{app:non-fixed-switching}
We start by explicitly writing $\tau(s) = \tau_{\theta}(s)$ and $Q^{\tau}(s, d(s)) = Q^{\tau}_{\theta}(s, d(s))$ to highlight the dependence with respect to the machine parameters $\theta$.
Moreover, under an $\epsilon$-greedy triage policy, we have that:
\begin{equation}
    \tau_{\theta}(s) = 
    \begin{cases}
    1 - \frac{\epsilon}{2} & \text{if } Q^{\tau}_{\theta}(s, 1) \leq Q^{\tau}_{\theta}(s, 0) \\
    \frac{\epsilon}{2} & \text{otherwise}
    \end{cases}
%    \implies \nabla_{\theta} \tau_{\theta}(s)  = 
%    \begin{cases}
%    \infty & \text{if }  Q^{\tau}_{\theta}(s, 1) = Q^{\tau}_{\theta}(s, 0) \\
%    0 & \text{otherwise}
%    \end{cases}
\end{equation}
Then, it follows that:
\begin{align*}
\frac{\partial v^{\tau}}{\partial \theta} &= \frac{\partial (\tau_{\theta}(s)  c_c(1) + (1-\tau_{\theta}(s) )c_c(0))}{\partial \theta} + \sum_{a \in \Acal}\frac{\partial (\tau_{\theta}(s)  \pi_{\MM,\theta}(a\given s) + (1-\tau_{\theta}(s) )\pi_{\HH}(a\given s)}{\partial \theta}q^{\tau}(s,a)   \\
&\qquad \qquad \qquad+\sum_{a \in \Acal}\left(\tau_{\theta}(s) \pi_{\MM,\theta}(a\given s) + (1-\tau_{\theta}(s) )\pi_{\HH}(a\given s)\right) \frac{\partial q^{\tau}(s,a)}{\partial \theta} \\
&= \frac{\partial \tau_{\theta}(s) }{\partial \theta}(c_c(1) - c_c(0)) + \sum_{a \in \Acal} \frac{\partial \tau_{\theta}(s)}{\partial \theta}(\pi_{\MM,\theta}(a\given s) - \pi_{\HH}(a\given s))q^{\tau}(s,a) \\
&\qquad \qquad \qquad+ \sum_{a \in \Acal} \tau_{\theta}(s) \frac{\partial \pi_{\MM,\theta}}{\partial \theta} q^{\tau}(s,a) + \sum_{a \in \Acal}\left(\tau_{\theta}(s) \pi_{\MM,\theta}(a\given s) + (1-\tau_{\theta}(s) )\pi_{\HH}(a\given s)\right)\frac{\partial q^{\tau}(s,a)}{\partial \theta} \\
&\ost{(i)}{=}
\frac{\partial \tau_{\theta}(s) }{\partial \theta}\left( Q^{\tau}_\theta(s, 1) - Q^{\tau}_\theta(s, 0) \right) + \sum_{a \in \Acal} \tau_{\theta}(s) \frac{\partial \pi_{\MM,\theta}}{\partial \theta} q^{\tau}(s,a)\\
&\qquad \qquad \qquad+\sum_{a \in \Acal}\left(\tau_{\theta}(s) \pi_{\MM,\theta}(a\given s) + (1-\tau_{\theta}(s) )\pi_{\HH}(a\given s)\right)\frac{\partial q^{\tau}(s,a)}{\partial \theta} 
\\ 
&\ost{(ii)}{=}
\sum_{a \in \Acal} \tau_{\theta}(s) \frac{\partial \pi_{\MM,\theta}}{\partial \theta} q^{\tau}(s,a)+\sum_{a \in \Acal}\left(\tau_{\theta}(s) \pi_{\MM,\theta}(a\given s) + (1-\tau_{\theta}(s) )\pi_{\HH}(a\given s)\right)\frac{\partial q^{\tau}(s,a)}{\partial \theta}
\\
&=
\sum_{a \in \Acal} \tau_{\theta}(s) \frac{\partial \pi_{\MM,\theta}}{\partial \theta} q^{\tau}(s,a)+\sum_{a \in \Acal}\left(\tau_{\theta}(s) \pi_{\MM,\theta}(a\given s) + (1-\tau_{\theta}(s) )\pi_{\HH}(a\given s)\right)\sum_{s' \in \mathcal{S}} p(s'\given s,a)\frac{\partial v^{\tau}(s')}{\partial \theta},
\end{align*}
where, in (i), we have used Eq.~\ref{eq:option-q} and, in (ii), we have used Lemma~\ref{lem:switching-gradient} below. Since the above 
expression matches Eq.~\ref{eq:nabla_v}, this concludes the proof.

\begin{lemma} \label{lem:switching-gradient}
Let $\tau(s) = \tau_{\theta}(s) $ and $Q^{\tau}(s, d(s)) = Q^{\tau}_{\theta}(s, d(s))$. For any $s \in \Scal$ and $\theta$, it 
holds that:
\begin{equation} \label{eq:switching-gradient}
    \frac{\partial \tau_{\theta}(s)}{\partial \theta}\left( Q^{\tau}_\theta(s, 1) - Q^{\tau}_\theta(s, 0) \right) = 0
\end{equation}
\end{lemma}
\begin{proof}
Let $\theta_0$ be the point where $Q^{\tau}_{\theta_0}(s, 1) = Q^{\tau}_{\theta_0}(s, 0)$. For $\theta \neq \theta_0$, Eq.~\ref{eq:switching-gradient} holds since $\nabla_{\theta} \tau_{\theta}(s)  = 0$. 
Now, assume that $Q^{\tau}_{\theta}(s,d(s))$ is continuous w.r.t. $\theta$ and that w.l.o.g that 
$Q^{\tau}_{\theta}(s, 1) > Q^{\tau}_{\theta_0}(s, 0)$, for $\theta = \theta_0 + \varepsilon$ and $Q^{\tau}_{\theta}(s, 1) < Q^{\tau}_{\theta_0}(s, 0)$, 
for $\theta = \theta_0 - \varepsilon$, for small $\varepsilon > 0$. Then we have that:
\begin{align*}
\lim_{\theta \rightarrow \theta_0^+} \frac{\tau_{\theta}(s)  - \tau_{\theta_0}(s) }{\theta - \theta_0} \left(Q^{\tau}_{\theta_0}(s, 1) - Q^{\tau}_{\theta_0}(s, 0)\right) & =  \lim_{\theta \rightarrow \theta_0^+} \frac{ \frac{\epsilon}{2} - (1 - \frac{\epsilon}{2})}{\theta - \theta_0} \left(Q^{\tau}_{\theta_0}(s, 1) - Q^{\tau}_{\theta_0}(s, 0)\right)\\
&= \lim_{\theta \rightarrow \theta_0^+}  \frac{ (\epsilon - 1)\left(Q^{\tau}_{\theta_0}(s, 1) - Q^{\tau}_{\theta_0}(s, 0)\right)}{\theta - \theta_0} \\ 
&\ost{(i)}{=} \lim_{\theta \rightarrow \theta_0^+}  \frac{(\epsilon - 1)\left(\nabla_{\theta} Q^{\tau}_{\theta_0}(s, 1) - \nabla_{\theta}Q^{\tau}_{\theta_0}(s, 0)\right)}{1}\\
&=\lim_{\theta \rightarrow \theta_0^+}\frac{0}{1} = 0,
\end{align*}
where, in $(i)$, we have used L'H\^opital's rule. Moreover, we also have that:
\begin{align*}
\lim_{\theta \rightarrow \theta_0^-} \frac{\tau_{\theta}(s)  - \tau_{\theta_0}(s) }{\theta - \theta_0} \left(Q^{\tau}_{\theta_0}(s, 1) - Q^{\tau}_{\theta_0}(s, 0)\right) & =  \lim_{\theta \rightarrow \theta_0^-} \frac{ (1-\frac{\epsilon}{2}) - (1 - \frac{\epsilon}{2})}{\theta - \theta_0} \left(Q^{\tau}_{\theta_0}(s, 1) - Q^{\tau}_{\theta_0}(s, 0)\right)\\
&= \lim_{\theta \rightarrow \theta_0^-}  \frac{ ((1-\frac{\epsilon}{2}) - (1 - \frac{\epsilon}{2}))\left(Q^{\tau}_{\theta_0}(s, 1) - Q^{\tau}_{\theta_0}(s, 0)\right)}{\theta - \theta_0} \\ 
&\ost{(i)}{=} \lim_{\theta \rightarrow \theta_0^-}  \frac{0\cdot\left(\nabla_{\theta} Q^{\tau}_{\theta_0}(s, 1) - \nabla_{\theta}Q^{\tau}_{\theta_0}(s, 0)\right)}{1}\\
&=\lim_{\theta \rightarrow \theta_0^-}\frac{0}{1} = 0,
\end{align*}
where, in $(i)$, we have also used L'H\^opital's rule. This concludes the proof.
\end{proof}

\section{Derivation of Matrix Equality} \label{app:matrix-A}
To show that $\Ab = \boldsymbol \Phi^T \boldsymbol F (I - \boldsymbol P) \boldsymbol \Phi$, we proceed as follows:
\begin{align*}
 \Ab &= \lim_{t \rightarrow \infty} \EE[\boldsymbol A_t]
 = \sum_{s} d_{\pi_{\HH}}(s) \EE[F_t \varrho_t \boldsymbol\phi(s,d_t)(\boldsymbol\phi(s,d_t) - \boldsymbol\phi(s,d_{t+1}))^T  \given  s_t {=}s] \\
 &= \sum_{s} d_{\pi_{\HH}}(s) \EE[F_t \given  s_t {=}s] \EE[\varrho_t \boldsymbol\phi(s,d_t)(\boldsymbol\phi(s,d_t) - \boldsymbol\phi(s,d_{t+1}))^T ] \\
 &= \sum_{s} f(s) \sum_{a}\pi_{\HH}(a\given s)\varrho(a,s)\EE_{d \sim  \tau(s)}[\boldsymbol\phi(s,d)](\EE_{d\sim \tau(s)}[\boldsymbol\phi(s,d)] - \sum_{s'}p(s'\given s,a)\EE_{d \sim \tau(s')}[\boldsymbol\phi(s',d)])^T \\
 &= \sum_{s} f(s) \sum_{a}\pi_{\HH}(a\given s) \frac{\varpi(a\given s)}{\pi_{\HH}(a\given s)}\EE_{d \sim \tau(s)}[\boldsymbol\phi(s,d)](\EE_{d\sim \tau(s)}[\boldsymbol\phi(s,d)] - \sum_{s'}p(s'\given s,a)\EE_{d \sim \tau(s')}[\boldsymbol\phi(s',d)])^T \\
 &= \sum_{s} f(s)\EE_{d\sim \tau(s)}[\boldsymbol\phi(s,d)](\EE_{d \sim  \tau(s)}[\boldsymbol\phi(s,d)] - \sum_{a}\varpi(a\given s)\sum_{s'}p(s'\given s,a)\EE_{d \sim \tau(s')}[\boldsymbol\phi(s',d)])^T \\
 &= \boldsymbol \Phi^T \boldsymbol F (I - \boldsymbol P) \boldsymbol \Phi.
 \end{align*}
  
\section{Proof of Claim~\ref{claim}}
We have that $\boldsymbol K$ is positive definite if $\boldsymbol S = \boldsymbol K + \boldsymbol K^T$ is  positive definite. 
Moreover, if $\boldsymbol S$ is strictly diagonally dominant~\footnote{\label{diagonally-dominant}The matrix $\boldsymbol A$ is strictly diagonally dominant if $|a_{ii}| > \sum_{j,j \neq i} |a_{ij}|$ for every  i}, then it is also positive definite. 
Given that the diagonal elements of $\boldsymbol K$ are positive and the off diagonal negative --see (i) below--, 
$\boldsymbol S$ is strictly diagonally dominant if each row sum plus the corresponding column sum of $\boldsymbol K$ 
is positive. 
This is because $S(i,j) = K(i,j) + K(j,i)$ if $i \neq j$ and $S(i,i) = 2K(i,i)$ so it must hold that: 
\begin{align*}
    |S(i,i)| > \sum_{i\neq j}|S(i,j)| 
    & \Rightarrow |2K(i,i)| > \sum_{i\neq j}|K(i,j) + K(j,i)| \\
    & \Rightarrow |2K(i,i)| - \sum_{i\neq j}|K(i,j) + K(j,i)| > 0 \\ 
    & \ost{(i)}{\Rightarrow} 2K(i,i) + \sum_{i\neq j} (K(i,j) + K(j,i)) > 0  \Rightarrow  \sum_{i,j}K(i,j) + \sum_{i,j}K(j,i) > 0
\end{align*}
This concludes the proof.

\section{Experiments: Additional Details} 
Here we provide additional details about the environment design and implementation of the methods.
%the  code  and  furtherimplementation details as part of the supplementary material,including the neural network architectures used for the actorand the critic in different methods

\subsection{Environment details} 
The sampling of each cell type during the environment generation depends on the traffic level of the corresponding row, following the distribution in Table~\ref{tab:type_distribution}. The traffic level of each row depends on the traffic level of the previous row according to the transition diagram in Figure~\ref{fig:transition-diagram}. Each episode starts at a random row with traffic level \texttt{light}.
\begin{figure}[h]
\begin{floatrow}
\capbtabbox{%
% \begin{table}[h]
\begin{tabular}{@{}lllll@{}}
\toprule
       & road & grass & stone & car \\ \midrule
no-car & 0.7  & 0.2   & 0.1   & 0   \\
low    & 0.6  & 0.2   & 0.1   & 0.1 \\
heavy  & 0.5  & 0.2   & 0.1   & 0.2 \\ \bottomrule
\end{tabular}
% \end{table}
}{%
  \caption{Cell type distribution }%
  \label{tab:type_distribution}
}
\ffigbox{%
\includegraphics[width=.8\linewidth]{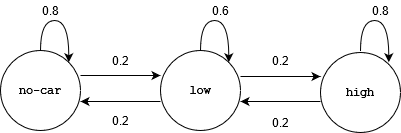}
% \end{figure}
}{%
  \caption{Traffic level transition diagram}%
  \label{fig:transition-diagram}
}
\end{floatrow}
\end{figure}

\subsection{Implementation details }
% \label{app:networks}
We describe below the network architectures as well as the hyperparameters and some implementation specific choices we made in our experiments. For the development of the code we used Python 3.7 and PyTorch 1.8.1. In all experiments we used fixed random seeds for reproducibility. We note also that even though all algorithms ran once for each scenario-- for time and computational reasons--, we observed consistency of the results across different scenarios. 

\xhdr{Networks}%
The machine neural policy has the following representation: (a) the input is $76$ binary features corresponding to the perceived state, (b) the output is the estimated log probability of each action, and (c) there is $1$ hidden layer with $256$ units and hyperbolic tangent activation. The option value function network has the following representation: (a) the input is $78$ binary features ($76$ for the state and additional $2$-bits encoding the agent in control), (b) the output is the estimated option value, and (c) there is $1$ hidden layer same as for the machine neural policy. 

\xhdr{Hyperparameters}%
To finalize the number of offline and online training episodes, we experimented with various combinations. For the offline stage we experimented with values in $[50,100]$ with the incentive to achieve minimum cost without overfitting in all methods for all scenarios, for comparison reasons. For the online stage we chose per scenario from values in $\{50, 100, 200\}$ with the incentive to reach convergence, while saving computational time in all methods. We tried also different schedules for $\epsilon$ and chose the one resulting in optimal performance. In the offline stage, except the schedule applied in final experiments, we experimented also with: a) a fixed $\epsilon \in \{0.1, 0.2, 0.3\}$, b) $\epsilon=0.3$ in the first half and $\epsilon=0.2$ in the second half. In the online stage we assumed only decaying $\epsilon$ while trying different decay rates $1/t, 1/\sqrt{t}$. We used the selected schedule for $\epsilon$ in all scenarios in all methods.
In all experiments we used RMSProp optimizer as used in the implementation\footnote{\label{ft:reference-code}\url{https://github.com/ShangtongZhang/DeepRL}} of ACE \cite{imani2019offpolicy} and COF-PAC \cite{pmlr-v119-zhang20s} and  empirically adjusted the initial learning rate to $10^{-4}$ after trying values in $\{10^{-3},10^{-4},10^{-5}\}$; our choice was based on stability and training efficiency. %time.
For simplicity of computation we used batch of size 1.

\xhdr{Implementation choices} In order to encourage exploration in the online stage we applied entropy regularization in the actor update with an initial weight of $0.01$, decaying with rate $1/t$. Moreover, for this update 
% at some time step $t$, in which the state $s_t$ and action $a_t$ were observed, 
in practice we used $c(s_t, a_t) + v_{\vartheta_t}^{\tau}(s_{t+1}) - v_{\vartheta_t}^{\tau}(s_{t})$ to approximate the true $v^{\tau}(s_t)$ that is required in Eq. \ref{eq:gradient-update-on-policy}, as this increased the stability of our algorithm. To this end, in all experiments, we also used a separate target network -- a frozen copy of the option value function network, updated every $5000$ steps-- for the computation of $Q^{\tau}_{\vartheta_t}(s_{t+1}, d_{t+1})$ as in the implementation\textsuperscript{\ref{ft:reference-code}}.